\documentclass{article}

\usepackage{microtype}
\usepackage{graphicx}
\graphicspath{{./},{./figures/}}

\usepackage{subfigure}
\usepackage{booktabs} %

\usepackage{hyperref}

\usepackage{amsmath}
\usepackage{amssymb}
\usepackage{comment}

\usepackage{cancel}
\usepackage{amsthm}

\usepackage{tikz}
\usetikzlibrary{positioning,fit, calc}
\usepackage{pgfplots}

\usepackage{wrapfig}
\usepackage{enumitem}

\usepackage{xcolor}
\usepackage{dsfont}

\usepackage[export]{adjustbox}

\usepackage{url}            %
\usepackage{amsfonts}       %

\usepackage{cleveref}
\usepackage[active]{srcltx}

\usepackage[algo2e,inoutnumbered,algoruled,vlined]{algorithm2e} %
\usepackage{wrapfig}
\usepackage{mathtools}

\usepackage{subfigure}

\usepackage{xargs}

\newcommand{\supp}{the supplementary document}

\newcommand{\rmd}{r}
\newcommand{\Sp}{\mathbb{S}}
\newcommand{\R}{\mathbb{R}}
\newcommand{\E}{\mathbb{E}}

\newcommand{\W}{{\cal W}_2}

\newcommand{\F}{{\cal F}}
\newcommand{\PS}{{\cal P}}
\newcommand{\He}{{\cal H}}
\newcommand{\SW}{{\cal S}{\cal W}_2}
\newcommand{\TV}{\textnormal{TV}}
\newcommand{\KL}{\textnormal{KL}}
\newcommand{\muh}{\hat{\mu}}
\newcommand{\mub}{\bar{\mu}}

\newtheorem{thm}{Theorem}

\newtheorem{cor}{Corollary}
\newtheorem{lemma}{Lemma}

\DeclareMathOperator*{\argmin}{arg\min}

\DeclareMathOperator{\cB}{\overline{B}}
\newtheorem{definition}{Definition}

\newcommand\simiid{\stackrel{\mathclap{\normalfont\mbox{\tiny i.i.d.}}}{\sim}}

\newtheorem{assumption}{\textbf{H}\hspace{-3pt}}
\Crefname{assumption}{\textbf{H}\hspace{-3pt}}{\textbf{H}\hspace{-3pt}}
\crefname{assumption}{\textbf{H}}{\textbf{H}}

\newcommand{\ps}[2]{\left\langle#1,#2 \right\rangle}

\newcommand{\ocint}[1]{\left(#1\right]}

\newcommand{\ccint}[1]{\left[#1\right]}
\def\eg{e.g.}

\def\mca{\mathcal{A}}

\def\mcf{\mathcal{F}}
\def\mcg{\mathcal{G}}
\def\mch{\mathcal{H}}

\def\rset{\mathbb{R}}

\def\nset{\mathbb{N}}

\def\rmd{\mathrm{d}}

\def\mrl{\mathrm{L}}

\def\rmc{\mathrm{C}}

\newcommandx{\norm}[2][1=]{\ifthenelse{\equal{#1}{}}{\left\Vert #2 \right\Vert}{\left\Vert #2 \right\Vert^{#1}}}
\newcommandx{\normLigne}[2][1=]{\ifthenelse{\equal{#1}{}}{\Vert #2 \Vert}{\Vert #2\Vert^{#1}}}

\def\plusinfty{+\infty}
\def\ie{\textit{i.e.}}

\def\divop{\operatorname{div}}

\usepackage[accepted]{icml2019}

\usepackage{xr}
\newcommand{\fint}{\int_{\Sp^{d-1}}}

\icmltitlerunning{Sliced-Wasserstein Flows}

\begin{document}

\twocolumn[
\icmltitle{Sliced-Wasserstein Flows: Nonparametric Generative Modeling via Optimal Transport and Diffusions}

\begin{icmlauthorlist}
\icmlauthor{Antoine Liutkus}{inria}
\icmlauthor{Umut \c{S}im\c{s}ekli}{telecool}
\icmlauthor{Szymon Majewski}{impan}
\icmlauthor{Alain Durmus}{ens}
\icmlauthor{Fabian-Robert St\"oter}{inria}
\end{icmlauthorlist}

\icmlaffiliation{inria}{Inria and LIRMM, Univ. of Montpellier, France}
\icmlaffiliation{telecool}{LTCI, T\'{e}l\'{e}com Paristech, Universit\'{e} Paris-Saclay, Paris, France }
\icmlaffiliation{impan}{Institute of Mathematics, Polish Academy of Sciences, Warsaw, Poland}
\icmlaffiliation{ens}{CNRS, ENS Paris-Saclay,Université Paris-Saclay, Cachan, France}

\icmlcorrespondingauthor{Antoine Liutkus}{antoine.liutkus@inria.fr}
\icmlcorrespondingauthor{Umut \c{S}im\c{s}ekli}{umut.simsekli@telecom-paristech.fr}

\icmlkeywords{Machine Learning, ICML, sliced wasserstein}

\vskip 0.3in
]

\printAffiliationsAndNotice{}  %

\begin{abstract}
By building upon the recent theory that established the connection between implicit generative modeling (IGM) and optimal transport, in this study, we propose a novel parameter-free algorithm for learning the underlying distributions of complicated datasets and sampling from them. The proposed algorithm is based on a functional optimization problem, which aims at finding a measure that is close to the data distribution as much as possible and also expressive enough for generative modeling purposes. We formulate the problem as a gradient flow in the space of probability measures. The connections between gradient flows and stochastic differential equations let us develop a computationally efficient algorithm for solving the optimization problem. We provide formal theoretical analysis where we prove finite-time error guarantees for the proposed algorithm. To the best of our knowledge, the proposed algorithm is the first nonparametric IGM algorithm with explicit theoretical guarantees. Our experimental results support our theory and show that our algorithm is able to successfully capture the structure of different types of data distributions.
\end{abstract}

\section{Introduction}

Implicit generative modeling (IGM) \cite{diggle1984monte, mohamed2016learning} has become very popular recently and has proven
successful in various fields; variational auto-encoders (VAE) \cite{kingma2013VAE} and generative adversarial networks (GAN) \cite{goodfellow2014generative} being its two well-known examples. The goal in IGM can be briefly described as learning the
underlying probability measure of a given dataset, denoted as $\nu \in \PS(\Omega)$, where $\PS$ is the space of probability measures on the measurable space $(\Omega,\mca)$, $\Omega \subset \rset^d$ is a domain and $\mca$ is the associated Borel $\sigma$-field.

Given a set of data points $\{y_1 , \dots , y_P \}$ that are assumed to be independent and identically distributed (i.i.d.) samples drawn from $\nu$, the implicit generative framework models them as the output of a measurable map, i.e.\ $y = T(x)$, with $T: \Omega_\mu \mapsto \Omega$. Here, the inputs $x$ are generated from a known and easy to sample source measure $\mu$ on $\Omega_\mu$ (e.g.\ Gaussian or uniform measures), and the outputs $T(x)$ should match the unknown target measure $\nu$ on $\Omega$.

Learning generative networks have witnessed several groundbreaking contributions in recent years. Motivated by this fact, there has been an interest in illuminating the theoretical foundations of VAEs and GANs \cite{bousquet2017optimal,liu2017approximation}.
It has been shown that these implicit models have close connections with the theory of Optimal Transport (OT) \cite{villani2008optimal}.
As it turns out, OT brings new light on the generative modeling problem: there have been several extensions of VAEs \cite{tolstikhin2017wasserstein,kolouri2018sliced} and GANs \cite{arjovsky2017wasserstein,gulrajani2017improved,guo2017relaxed,lei2017geometric}, which exploit the links between OT and IGM.

OT studies whether it is possible to transform samples from a source distribution $\mu$ to a target distribution $\nu$. From this perspective, an ideal generative model is simply a transport map from $\mu$ to $\nu$.
This can be written by using some `push-forward operators': we seek a mapping $T$ that `pushes $\mu$ onto $\nu$', and is formally defined as $ \nu(A) =  \mu(T^{-1}(A)) $ for all Borel sets $A \subset \mca$. If this relation holds, we denote the push-forward operator $T_\#$, such that $T_\# \mu = \nu$. Provided mild conditions on these distributions hold (notably $\mu$ is non-atomic \cite{villani2008optimal}), existence of such a transport map is guaranteed; however, it remains a challenge to construct it in practice.

One common point between VAE and GAN is to adopt an approximate strategy and consider transport maps that belong to a \emph{parametric} family $T_{\phi}$ with $\phi \in \Phi$. Then, they aim at finding the best parameter $\phi^\star$ that would give $T_{\phi^\star \#}\mu \approx \nu$. This is typically achieved by attempting to minimize the following optimization problem:
$\phi^\star = \argmin_{\phi \in \Phi} \W(T_{\phi \#}\mu, \nu)$,
where $\W$ denotes the Wasserstein distance that will be properly defined in Section~\ref{sec:techbg}. It has been shown that \cite{genevay2017gan} OT-based GANs \cite{arjovsky2017wasserstein} and VAEs \cite{tolstikhin2017wasserstein} both use this formulation with different parameterizations and different equivalent definitions of $\W$. However, their resulting algorithms still lack theoretical understanding.

In this study, we follow a completely different approach for IGM, where we aim at developing an algorithm with explicit theoretical guarantees for estimating a transport map between source $\mu$ and target $\nu$. The generated transport map  will be \textit{nonparametric} (in the sense that it does not belong to some family of functions, like a neural network), and it will be iteratively augmented: always increasing the quality of the fit along iterations. Formally, we take $T_t$ as the constructed transport map at time $t \in [0,\infty)$, and define $\mu_t=T_t \# \mu$ as the corresponding output distribution. Our objective is to build the maps so that $\mu_t$ will converge to the solution of a functional optimization problem, defined through a gradient flow in the Wasserstein space. Informally, we will consider a gradient flow that has the following form:
\begin{align}
\partial_t \mu_t = - \nabla_{\W} \Bigl\{ \mathrm{Cost}(\mu_t, \nu) + \mathrm{Reg}(\mu_t)\Bigr\} \, , \>\> \mu_0 = \mu,\label{eqn:gradflow}
\end{align}
where the functional $\mathrm{Cost}$ computes a discrepancy between $\mu_t$ and $\nu$, $\mathrm{Reg}$ denotes a regularization functional, and $\nabla_{\W}$ denotes a notion of gradient with respect to a probability measure in the $\W$ metric for probability measures\footnote{This gradient flow is similar to the usual Euclidean gradient flows, i.e.\ $\partial_t x_t = - \nabla (f(x_t) + r(x_t))$, where $f$ is typically the data-dependent cost function and $r$ is a regularization term. The (explicit) Euler discretization of this flow results in the well-known gradient descent algorithm for solving $\min_x (f(x)+r(x))$.}. If this flow can be simulated, one would hope for $\mu_t=(T_t)_{\#}\mu$ to converge to the minimum of the functional optimization problem: $\min_\mu ( \mathrm{Cost}(\mu, \nu) + \mathrm{Reg}(\mu))$ \cite{ambrosio2008gradient,santambrogio2017euclidean}.

We construct a gradient flow where we choose the $\mathrm{Cost}$ functional as the \textit{sliced Wasserstein distance} ($\SW$) \cite{rabin:et:al:2011,bonneel2015sliced} and the $\mathrm{Reg}$ functional as the negative entropy. The $\SW$ distance is equivalent to the $\W$ distance \cite{bonnotte2013unidimensional} and has important computational implications since it can be expressed as an average of (one-dimensional) projected optimal transportation costs whose analytical expressions are available.

We first show that, with the choice of $\SW$ and the negative-entropy functionals as the overall objective, we obtain a valid gradient flow that has a solution path $(\mu_t)_t$, and the probability density functions of this path solve a particular partial differential equation, which has close connections with stochastic differential equations. Even though gradient flows in Wasserstein spaces cannot be solved in general, by exploiting this connection, we are able to develop a practical algorithm that provides approximate solutions to the gradient flow and is algorithmically similar to stochastic gradient Markov Chain Monte Carlo (MCMC) methods\footnote{We note that, despite the algorithmic similarities, the proposed algorithm is not a Bayesian posterior sampling algorithm.} \cite{WelTeh2011a,ma2015complete,durmus2016stochastic,csimcsekli2017fractional,pmlr-v80-simsekli18a}. We provide finite-time error guarantees for the proposed algorithm and show explicit dependence of the error to the algorithm parameters.

To the best of our knowledge, the proposed algorithm is the first nonparametric IGM algorithm that has explicit theoretical guarantees. In addition to its nice theoretical properties, the proposed algorithm has also significant practical importance: it has low computational requirements and can be easily run on an everyday laptop CPU.%
Our experiments on both synthetic and real datasets support our theory and illustrate the advantages of the algorithm in several scenarios.

\section{Technical Background}
\label{sec:techbg}

\vspace{-2pt}

\subsection{Wasserstein distance, optimal transport maps and Kantorovich potentials }
For two probability measures $\mu,\nu \in \PS_2(\Omega)$, $\PS_2(\Omega) = \{ \mu \in \PS(\Omega) \, :\, \int_{\Omega} \norm[2]{x} \mu(\rmd x) < \plusinfty\}$, the 2-Wasserstein distance is defined as follows:
\begin{align}
\W(\mu,\nu) \triangleq \Bigl\{ \inf_{\gamma \in {\cal C}(\mu,\nu)} \int_{\Omega \times \Omega} \|x-y\|^2 \gamma(dx , dy) \Bigr\}^{1/2}, \label{eqn:w2}
\end{align}
where ${\cal C}(\mu,\nu)$ is called the set of \emph{transportation plans} and defined as the set of probability measures $\gamma$ on $\Omega \times \Omega$ satisfying for all $A \in {\cal A}$, $\gamma(A \times \Omega) = \mu(A)$ and $\gamma(\Omega \times A)=\nu(A)$, i.e. the  marginals of $\gamma$  coincide with $\mu$ and $\nu$. From now on, we will assume that $\Omega$ is a compact subset of $\R^d$.

In the case where $\Omega$ is finite, computing the Wasserstein distance between two probability measures turns out to be  a linear program with linear constraints, and has therefore a dual formulation. Since $\Omega$ is a Polish space (i.e.\ a complete and separable metric space), this dual formulation can be generalized as follows \cite{villani2008optimal}[Theorem 5.10]:
\begin{align}
\W(\mu,\nu) \hspace{-1pt} = \hspace{-6pt} \sup_{\psi \in \mathrm{L}^1(\mu)} \Bigl\{ \int_\Omega \psi(x) \mu(dx) + \int_\Omega \psi^c(x) \nu(dx) \Bigr\}^{1/2} \label{eqn:w2dual}
\end{align}
where $\mathrm{L}^1(\mu)$ denotes the class of functions that are absolutely integrable under $\mu$ and $\psi^c$ denotes the c-conjugate of $\psi$ and is defined as follows: $\psi^c(y) \triangleq \{ \inf_{x\in \Omega} \| x-y\|^2 - \psi(x)\}$. The functions $\psi$ that realize the supremum in \eqref{eqn:w2dual} are called the Kantorovich potentials between $\mu$ and $\nu$.
Provided that $\mu$ satisfies a mild condition, we have the following  uniqueness result.
\begin{thm}[\protect{\cite{santambrogio2010introduction}[Theorem 1.4]}]
\label{thm:unqmap}
Assume that  $\mu\in \PS_2(\Omega)$ is absolutely continuous with respect to the Lebesgue measure. Then, there exists a unique optimal transport plan $\gamma^\star$ that realizes the infimum in \eqref{eqn:w2} and it is of the form $(\text{Id} \times T)_\# \mu$, for a measurable function $T : \Omega \to \Omega$. Furthermore, there exists at least a Kantorovich potential $\psi$ whose gradient $\nabla \psi$ is uniquely determined $\mu$-almost everywhere. The function $T$ and the potential $\psi$ are linked by $T(x) = x- \nabla \psi(x)$.
\end{thm}
The measurable function $T : \Omega \to \Omega$ is referred to as the optimal transport map from $\mu$ to $\nu$.
This result implies that there exists a solution for transporting samples from $\mu$ to samples from $\nu$ and this solution is optimal in the sense that it minimizes the $\ell_2$ displacement. However, identifying this solution is highly non-trivial. In the discrete case, effective solutions have been proposed \cite{cuturi2013sinkhorn}. However, for continuous and high-dimensional probability measures, constructing an actual transport plan remains a challenge. Even if recent contributions \cite{genevay2016stochastic} have made it possible to rapidly compute $\W$, they do so without constructing the optimal map $T$, which is our objective here.

\subsection{Wasserstein spaces and gradient flows}

By \cite{ambrosio2008gradient}[Proposition 7.1.5], $\W$ is a distance over $\PS(\Omega)$.
In addition, if $\Omega \subset \R^d$ is compact, the topology associated with $\W$ is equivalent to the weak convergence of probability measures and $(\PS(\Omega),\W)$\footnote{Note that in that case, $\PS_2(\Omega)=\PS(\Omega)$} is compact. The metric space $(\PS_2(\Omega),\W) $ is called the \emph{Wasserstein space}.

In this study, we are interested in functional optimization problems in $(\PS_2(\Omega),\W)$, such as $\min_{\mu\in\PS_2(\Omega)} \F(\mu)$, where $\F$ is the functional that we would like to minimize. Similar to Euclidean spaces, one way to formulate this optimization problem is to construct a gradient flow of the form $\partial_t \mu_t = - \nabla_{\W} \F(\mu_t)$ \cite{benamou2000computational,lavenant2018dynamical}, where $\nabla_{\W}$ denotes a notion of gradient in $(\PS_2(\Omega),\W)$. If such a flow can be constructed, one can utilize it both for practical algorithms and theoretical analysis.

Gradient flows $\partial_t \mu_t = \nabla_{\W} \mathcal{F}(\mu_t)$ with respect to a functional $\mathcal{F}$ in $(\PS_2(\Omega),\W)$ have strong connections with partial differential equations (PDE) that are of the form of a \emph{continuity equation} \cite{santambrogio2017euclidean}. Indeed, it is shown than under appropriate conditions on $\mathcal{F}$ (see \eg \cite{ambrosio2008gradient}), $(\mu_t)_t$ is a solution of the gradient flow if and only if it admits a density $\rho_t$ with respect to the Lebesgue measure for all $t \geq 0$, and solves the continuity equation given by:
$\partial_t \rho_t + \divop (v \rho_t) = 0$, %
where $v$ denotes a vector field and $\divop$ denotes the divergence operator. Then, for a given gradient flow in $(\PS_2(\Omega),\W)$, we are interested in the evolution of the densities $\rho_t$, i.e.\ the PDEs which they solve.
Such PDEs are of our particular interest since they have a key role for building practical algorithms.

\subsection{Sliced-Wasserstein distance}

In the one-dimensional case, i.e.\ $\mu,\nu \in \PS_2(\R)$, $\W$ has an analytical form, given as follows:
$\W(\mu,\nu) = \int_0^1 |F_\mu^{-1}(\tau) - F_\nu^{-1}(\tau)|^2 \> d\tau$, %
where $F_\mu$ and $F_\nu$ denote the cumulative distribution functions (CDF) of $\mu$ and $\nu$, respectively, and $F^{-1}_\mu, F^{-1}_\nu$ denote the inverse CDFs, also called quantile functions (QF).
In this case, the optimal transport map from $\mu$ to $\nu$  has a closed-form formula as well, given as follows: $T(x) = (F_\nu^{-1} \circ F_\mu) (x)$ \cite{villani2008optimal}. The optimal map $T$ is also known as the \emph{increasing arrangement}, which maps each quantile of $\mu$ to the same quantile of $\nu$, e.g. minimum to minimum, median to median, maximum to maximum \cite{villani2008optimal}.
Due to Theorem~\ref{thm:unqmap}, the derivative of the corresponding Kantorovich potential is given as:
\begin{align*}
\psi'(x) \triangleq \partial_x \psi(x) = x- (F_\nu^{-1} \circ F_\mu) (x).
\end{align*}

In the multidimensional case $d > 1$, building a transport map is much more difficult. The nice properties of the one-dimensional Wasserstein distance motivate the usage of \emph{sliced-Wasserstein distance} ($\SW$) for practical applications. Before formally defining $\SW$, let us first define the orthogonal projection $\theta^* (x) \triangleq \langle \theta, x \rangle$ for any direction $\theta \in \Sp^{d-1}$ and $x \in \R^d$, where $\langle \cdot, \cdot \rangle$ denotes the Euclidean inner-product and $\Sp^{d-1} \subset \R^d$ denotes the $d$-dimensional unit sphere. Then, the $\SW$ distance is formally defined as follows:
\begin{align}
\SW(\mu,\nu) \triangleq \int_{\Sp^{d-1}} \W (\theta^*_\#\mu, \theta^*_\#\nu) \> d \theta, \label{eqn:sw}
\end{align}
where $d\theta$ represents the uniform probability measure on $\Sp^{d-1}$. As shown in \cite{bonnotte2013unidimensional}, $\SW$ is indeed a distance metric and induces the same topology as $\W$ for compact domains.

The $\SW$ distance has important practical implications: provided that the projected distributions $\theta^*_\#\mu$ and $\theta^*_\#\nu$ can be computed, then for any $\theta \in \Sp^{d-1}$, the distance $\W (\theta^*_\#\mu, \theta^*_\#\nu)$, as well as its optimal transport map and the corresponding Kantorovich potential can be analytically computed (since the projected measures are one-dimensional). Therefore, one can easily approximate \eqref{eqn:sw} by using a simple Monte Carlo scheme that draws uniform random samples from $\Sp^{d-1}$ and replaces the integral in \eqref{eqn:sw} with a finite-sample average. Thanks to its computational benefits, $\SW$ was very recently considered for OT-based VAEs and GANs \cite{deshpande2018generative,autotranspoter,kolouri2018sliced}, appearing as a stable alternative to the adversarial methods.

\section{Regularized Sliced-Wasserstein Flows for Generative Modeling}

\subsection{Construction of the gradient flow}

In this paper, we propose the following functional minimization problem on $\PS_2(\Omega)$ for implicit generative modeling:
\begin{equation}
 \min_{\mu} \Bigl\{ \F^{\nu}_\lambda(\mu) \triangleq  \frac1{2} \SW^2(\mu, \nu) + \lambda \He(\mu) \Bigr\},  \label{eqn:sw_optim}
\end{equation}
where $\lambda >0$ is a regularization parameter and $\He$ denotes the negative entropy defined by $\He(\mu) \triangleq \int_{\Omega} \rho(x) \log \rho(x) dx $ if $\mu$ has density $\rho$ with respect to the Lebesgue measure and $\He(\mu) = + \infty$ otherwise. Note that the case $\lambda =0$ has been already proposed and studied in \cite{bonnotte2013unidimensional} in a more general OT context. Here, in order to introduce the necessary noise inherent to generative model, we suggest to penalize the slice-Wasserstein distance using $\He$. In other words, the main idea is to find a measure $\mu^\star$ that is close to $\nu$ as much as possible and also has a certain amount of entropy to make sure that it is sufficiently expressive for generative modeling purposes.
The importance of the entropy regularization becomes prominent in practical applications where we have finitely many data samples that are assumed to be drawn from $\nu$. In such a circumstance, the regularization would prevent $\mu^\star$ to collapse on the data points and therefore avoid `over-fitting' to the data distribution. Note that this regularization is fundamentally different from the one used in Sinkhorn distances \cite{genevay2018learning}.

In our first result, we show that there exists a flow $(\mu_t)_{t\geq0}$ in $(\PS(\cB(0,r)),\W)$ which decreases along $\F_\lambda^\nu$, where $\cB(0,a)$ denotes the closed unit ball centered at $0$ and radius $a$. This flow will be referred to as a generalized minimizing movement scheme (see Definition~$1$ in \supp).  In addition, the flow $(\mu_t)_{t \geq 0}$ admits a density $\rho_t$ with respect to the Lebesgue measure for all $t>0$ and $(\rho_t)_{t \geq 0}$ is solution of a non-linear PDE (in the weak sense). %
\begin{thm}
\label{thm:continuity}
Let $\nu$ be a probability measure on $\cB(0,1)$ with a strictly positive smooth density. Choose a regularization constant $\lambda > 0$ and radius $r > \sqrt{d}$, where $d$ is the data dimension. Assume that $\mu_0 \in \mathcal{P}(\cB(0,r))$ is absolutely continuous with respect to the Lebesgue measure with density $\rho_0 \in \mrl^{\infty}(\cB(0,r))$. There exists a generalized minimizing movement scheme  $(\mu_t)_{t \geq 0}$ associated to \eqref{eqn:sw_optim}
and if $\rho_t$ stands for the density of $\mu_t$ for all $t \geq 0$, then $(\rho_t)_t$ satisfies the following continuity equation:
\begin{align}
\frac{\partial \rho_t}{\partial t}   &= -\divop (v_t \rho_t) + \lambda \Delta \rho_t, \label{eqn:gradflow_reg} \\
v_t(x) \triangleq v(x,\mu_t) &= - \int_{\Sp^{d-1}} \psi_{t, \theta}'(\langle x , \theta \rangle ) \theta d\theta  \label{eqn:gradflow_reg_drift}
\end{align}
in a weak sense. Here, $\Delta$ denotes the Laplacian operator, $\divop$ the divergence operator, and $\psi_{t,\theta}$ denotes the Kantorovich potential between $\theta^*_{\#}\mu_t$ and $\theta^*_{\#}\nu$.
\end{thm}
The precise statement of this Theorem, related results and its proof are postponed to \supp. For its proof, we use the technique introduced in \cite{jordan1998variational}: we first prove the existence of a generalized minimizing movement scheme by showing that the solution curve $(\mu_t)_t$ is a limit of the solution of a time-discretized problem. Then we prove that the curve $(\rho_t)_t$ solves the PDE given in \eqref{eqn:gradflow_reg}.

\subsection{Connection with stochastic differential equations}

As a consequence of the entropy regularization, we obtain the Laplacian operator $\Delta$ in the PDE given in \eqref{eqn:gradflow_reg}. We therefore observe that the overall PDE is a Fokker-Planck-type equation \cite{bogachev2015fokker} that has a well-known probabilistic counterpart, which can be expressed as a stochastic differential equation (SDE). More precisely, let us consider a stochastic process $(X_t)_{t}$, that is the solution of the following SDE starting at $X_0 \sim \mu_0$:
\begin{align}
d X_t = v(X_t,\mu_t) dt + \sqrt{2 \lambda } d W_t, \label{eqn:sde}
\end{align}
where $(W_t)_t$ denotes a standard Brownian motion. Then, the probability distribution of $X_t$ at time $t$ solves the PDE given in \eqref{eqn:gradflow_reg} \cite{bogachev2015fokker}. This informally means that, if we could simulate \eqref{eqn:sde}, then the distribution of $X_t$ would converge to the solution of \eqref{eqn:sw_optim}, therefore, we could use the sample paths $(X_t)_t$ as samples drawn from $(\mu_t)_t$. However, in practice this is not possible due to two reasons: (i) the drift $v_t$ cannot be computed analytically since it depends on the probability distribution of $X_t$, (ii) the SDE \eqref{eqn:sde} is a continuous-time process, it needs to be discretized.

We now focus on the first issue.
We observe that the SDE \eqref{eqn:sde} is similar to McKean-Vlasov SDEs \cite{veretennikov2006ergodic,mishura2016existence}, a family of SDEs whose drift depends on the distribution of $X_t$. By using this connection, we can borrow tools from the relevant SDE literature \cite{malrieu03,cgm-08} for developing an approximate simulation method for \eqref{eqn:sde}.

Our approach is based on defining a \emph{particle system} that serves as an approximation to the original SDE \eqref{eqn:sde}. The particle system can be written as a collection of SDEs, given as follows \cite{bossy1997stochastic}:
\begin{align}
d X_t^i = v(X_t^i, \mu_t^{N}) dt + \sqrt{2 \lambda } d W_t^i \> , \quad i = 1,\dots, N, \label{eqn:sde_particle}
\end{align}
where $i$ denotes the particle index, $N \in \mathbb{N}_+$ denotes the total number of particles, and $\mu_t^N = (1/N) \sum_{j=1}^N \delta_{X_t^j}$ denotes the empirical distribution of the particles $\{X_t^j\}_{j=1}^N$. This particle system is particularly interesting, since (i) one typically has $\lim_{N \rightarrow \infty} \mu_t^{N}= \mu_t $ with a rate of convergence of order ${\cal O}(1/\sqrt{N})$ for all $t$ \cite{malrieu03,cgm-08}, and (ii) each of the particle systems in \eqref{eqn:sde_particle} can be simulated by using an Euler-Maruyama discretization scheme. We note that the existing theoretical results in \cite{veretennikov2006ergodic,mishura2016existence} do not directly apply to our case due to the non-standard form of our drift. However, we conjecture that a similar result holds for our problem as well. Such a result would be proven by using the techniques given in \cite{zhang2018stochastic}; however, it is out of the scope of this study.

\subsection{Approximate Euler-Maruyama discretization}
In order to be able to simulate the particle SDEs \eqref{eqn:sde_particle} in practice, we propose an approximate Euler-Maruyama discretization for each particle SDE.
The algorithm iteratively applies the following update equation: ($\forall i \in  \{1,\dots,N\}$)
\begin{align}
\bar{X}^i_0 \simiid \mu_0, \>\> \bar{X}^i_{k+1} = \bar{X}^i_k + h \hspace{0.5pt} \hat{v}_k(\bar{X}^i_k) + \sqrt{2 \lambda h} Z^i_{k+1}, \label{eqn:euler_particle}
\end{align}
where $k \in \mathbb{N}_+$ denotes the iteration number, $Z^i_k$ is a standard Gaussian random vector in $\R^d$, $h$ denotes the step-size, and $\hat{v}_k$ is a short-hand notation for a computationally tractable estimator of the original drift $v(\cdot, \bar{\mu}_{kh}^N)$, with $\bar{\mu}_{kh}^{N} = (1/N) \sum_{j=1}^N \delta_{\bar{X}_k^j}$ being the empirical distribution of $\{\bar{X}_k^j\}_{j=1}^N$. A question of fundamental practical importance is how to compute this function $\hat{v}$.

We propose to approximate the integral in \eqref{eqn:gradflow_reg_drift} via a simple Monte Carlo estimate.
This is done by first drawing $N_\theta$ uniform i.i.d.\ samples from the sphere $\Sp^{d-1}$, $\{\theta_{n}\}_{n=1}^{N_\theta}$. Then, at each iteration $k$, we compute:
\begin{align}
\hat{v}_k(x) \triangleq - (1/{N_\theta}) \sum\nolimits_{n=1}^{N_\theta} \psi_{k, \theta_{n}}'(\langle\theta_{n},x\rangle ) \theta_{n}, \label{eqn:approxdrift}
\end{align}
where for any $\theta$, $\psi_{k, \theta}'$ is the derivative of the Kantorovich potential (cf.\ Section~\ref{sec:techbg}) that is applied to the OT problem from $\theta^*_\#\bar{\mu}_{kh}^{N}$ to $\theta^*_\#\nu$: i.e.\,
\begin{align}
   \psi_{k, \theta}'(z) = \bigl[ z - (F^{-1}_{\theta^*_\#\nu} \circ F_{\theta^*_\#\bar{\mu}_{kh}^{N}}) (z)  \bigr] \label{eq:psiprime}.%
 \end{align}

     \begin{algorithm2e}[t]
         \SetInd{0.1ex}{1.5ex}
         \DontPrintSemicolon
         \SetKwInOut{Input}{input}
         \SetKwInOut{Output}{output}
         \Input{${\cal D} \equiv \{y_i\}_{i=1}^P$, $\mu_0$, $N$, $N_\theta$, $h$, $\lambda$}
         \Output{$\{\bar{X}_K^i\}_{i=1}^N$}
         {\color{purple} \small \tcp{Initialize the particles}}
         $\bar{X}_0^i \simiid \mu_0$, \hfill $i = 1,\dots,N$\\
     {\color{purple} \small \tcp{Generate random directions}}
     $\theta_{n} \sim \mathrm{Uniform}(\Sp^{d-1})$, \hfill $n = 1,\dots,N_\theta$\\
     {\color{purple} \small \tcp{Quantiles of projected target}}
     \For{$\theta\in\{\theta_{n}\}_{n=1}^{N_\theta}$}
     {
     $F^{-1}_{\theta^*_\#\nu}=\textnormal{QF}\{\langle\theta,y_i \rangle\}_{i=1}^P$\\
     }
     {\color{purple} \small \tcp{Iterations}}
     \For{$k = 0,\dots K-1$}
     {
        \For{$\theta\in\{\theta_{n}\}_{n=1}^{N_\theta}$}
        {
        {\color{purple} \small \tcp{CDF of projected particles}}
        $F_{\theta^*_\#\bar{\mu}_{kh}^{N}}=\textnormal{CDF}\{\langle\theta,\bar{X}_k^i \rangle\}_{i=1}^N$\\
        }
        {\color{purple} \small \tcp{Update the particles}}
        $\bar{X}_{k+1}^i = \bar{X}_{k}^i - h \hspace{0.5pt} \hat{v}_k(\bar{X}^i_k) + \sqrt{2 \lambda h} Z^i_{k+1}$ \vspace{2pt} \\
        $\hfill i = 1,\dots,N$
     }
         \caption{Sliced-Wasserstein Flow (SWF)}
         \label{algo:flow}
     \end{algorithm2e}

For any particular $\theta\in\Sp^{d-1}$, the QF, $F_{\theta^*_\#\nu}^{-1}$ for the projection of the target distribution $\nu$ on $\theta$ can be easily computed from the data. This is done by first computing the projections $\langle \theta, y_i\rangle$ for all data points $y_i$, and then computing the empirical quantile function for this set of $P$ scalars.
Similarly, $F_{\theta^*_\#\bar{\mu}_{kh}^{N}}$, the CDF of the particles at iteration $k$, is easy to compute: we first project all particles $\bar{X}_k^i$ to get $\langle \theta, \bar{X}_k^i\rangle$, and then compute the empirical CDF of this set of $N$ scalar values.

In both cases, the true CDF and quantile functions are approximated as a linear interpolation between a set of the computed $Q\in\mathbb{N}_+$ empirical quantiles.
Another source of approximation here comes from the fact that the target $\nu$ will in practice be a collection of Dirac measures on the observations $y_i$. Since it is currently common to have a very large dataset, we believe this approximation to be accurate in practice for the target.
Finally, yet another source of approximation comes from the error induced by using a finite number of $\theta_n$ instead of a sum over $\Sp^{d-1}$ in~\eqref{eq:psiprime}.

Even though the error induced by these approximation schemes can be incorporated into our current analysis framework, we choose to neglect it for now, because (i) all of these one-dimensional computations can be done very accurately and (ii) the quantization of the empirical CDF and QF can be modeled as additive Gaussian noise that enters our discretization scheme \eqref{eqn:euler_particle} \cite{van1998asymptotic}.
Therefore, we will assume that $\hat{v}_k$ is an \emph{unbiased} estimator of $v$, i.e.\ $\E[\hat{v}(x,\mu)] = v(x,\mu)$, for any $x$ and $\mu$, where the expectation is taken over $\theta_{n}$.

The overall algorithm is illustrated in Algorithm~\ref{algo:flow}. It is remarkable that the updates of the particles only involves the learning data $\{y_i\}$ through the CDFs of its projections on the many $\theta_{n}\in\Sp^{d-1}$. This has a fundamental consequence of high practical interest: these CDF may be computed beforehand in a massively distributed manner that is independent of the sliced Wasserstein flow. This aspect is reminiscent of the \textit{compressive learning} methodology \cite{gribonval2017compressive}, except we exploit quantiles of random projections here, instead of random generalized moments as done there.

Besides, we can obtain further reductions in the computing time if the CDF, $F_{\theta^*_\#\nu}$ for the target is computed on random mini-batches of the data, instead of the whole dataset of size $P$. This simplified procedure might also have some interesting consequences in privacy-preserving settings: since we can vary the number of projection directions $N_\theta$ for each data point $y_i$, we may guarantee that $y_i$ cannot be recovered via these projections, by picking fewer than necessary for reconstruction using, e.g. compressed sensing~\cite{donoho2009observed}.

\subsection{Finite-time analysis for the infinite particle regime}
In this section we will analyze the behavior of the proposed algorithm in the asymptotic regime where the number of particles $N \rightarrow \infty$. Within this regime, we will assume that the original SDE \eqref{eqn:sde} can be directly simulated by using an approximate Euler-Maruyama scheme, defined starting at $\bar{X}_0 \simiid \mu_0$ as follows:
\begin{align}
 \bar{X}_{k+1} = \bar{X}_k + h \hspace{0.5pt} \hat{v}(\bar{X}^i_k, \bar{\mu}_{kh} ) + \sqrt{2 \lambda h} Z_{k+1}, \qquad \label{eqn:euler_asymp}
\end{align}
where $\mub_{kh}$ denotes the law of $\bar{X}_k$ with step size $h$ and $\{Z_k\}_{k}$ denotes a collection of standard Gaussian random variables. Apart from its theoretical significance, this scheme is also practically relevant, since one would expect that it captures the behavior of the particle method \eqref{eqn:euler_particle} with large number of particles. 

In practice, we would like to approximate the measure sequence $(\mu_t)_t$ as accurate as possible, where $\mu_t$ denotes the law of $X_t$. Therefore, we are interested in analyzing the distance $\| \mub_{Kh} - \mu_{T} \|_{\TV}$, where $K$ denotes the total number of iterations, $T=Kh$ is called the horizon, and $\|\mu-\nu\|_{\TV}$ denotes the total variation distance between two probability measures $\mu$ and $\nu$: 
$\|\mu-\nu\|_{\TV}\triangleq \sup_{A \in {\cal B}(\Omega)} |\mu(A) -\nu(A) |$.

In order to analyze this distance, we exploit the algorithmic similarities between \eqref{eqn:euler_asymp} and the stochastic gradient Langevin dynamics (SGLD) algorithm \cite{WelTeh2011a}, which is a Bayesian posterior sampling method having a completely different goal, and is obtained as a discretization of an SDE whose drift has a much simpler form. We then bound the distance by extending the recent results on SGLD \cite{raginsky17a} to time- and measure-dependent drifts, that are of our interest in the paper.

We now present our second main theoretical result. We present all our assumptions and the explicit forms of the constants in \supp. 
\begin{thm}
\label{thm:euler}
Assume that the conditions given in \supp{} hold. Then, the following bound holds for $T=Kh$:
\begin{align}
\nonumber \| \mub_{Kh} - \mu_T \|_{\TV}^2 \leq \delta_\lambda \Biggl\{  \frac{L^2 K}{2\lambda} \Bigl( \frac{C_1 h^3}{3} + 3 \lambda d h^2 \Bigr) \hspace{13pt} \\ + \frac{C_2  \delta K h}{4\lambda} \Biggr\}	,
\end{align} 
for some $C_1,C_2,L >0$, $\delta \in (0,1)$, and $\delta_\lambda >1$.  %
\end{thm}
Here, the constants $C_1$, $C_2$, $L$ are related to the regularity and smoothness of the functions $v$ and $\hat{v}$; $\delta$ is directly proportional to the variance of $\hat{v}$, and $\delta_\lambda$ is inversely proportional to $\lambda$. The theorem shows that 
if we choose $h$ small enough, we can have a non-asymptotic error guarantee, which is formally shown in the following corollary. 
\begin{cor}
  \label{coro:precision}
  Assume that the conditions of Theorem~\ref{thm:euler} hold. Then for all $\varepsilon >0$, $K \in \mathbb{N}_+$, setting
  \begin{align}
h = (3/C_1)\wedge\left(\frac{2 \varepsilon^2 \lambda}{\delta_\lambda L^2 T}(1+3\lambda d)^{-1}\right)^{1/2}, %
  \end{align}
  we have
  \begin{align}
    \| \mub_{Kh} - \mu_T \|_{\TV} \leq \varepsilon + \left(\frac{C_2 \delta_\lambda \delta T}{4\lambda}\right)^{1/2} 
  \end{align}
  for $T=Kh$.
\end{cor}
This corollary shows that for a large horizon $T$, the approximate drift $\hat{v}$ should have a small variance in order to obtain accurate estimations. This result is similar to \cite{raginsky17a} and \cite{nguyen2019non}: for small $\varepsilon$ the variance of the approximate drift should be small as well. On the other hand, we observe that the error decreases as $\lambda$ increases. This behavior is expected since for large $\lambda$, the Brownian term in \eqref{eqn:sde} dominates the drift, which makes the simulation easier.

We note that these results establish the explicit dependency of the error with respect to the algorithm parameters (e.g. step-size, gradient noise) for a fixed number of iterations, rather than explaining the asymptotic behavior of the algorithm when $K$ goes to infinity.

\begin{figure*}
	\begin{tikzpicture}[font=\small]

		\begin{axis}[%
			/pgf/number format/.cd,
					use comma,
					1000 sep={},
		width=1\textwidth,
		axis line style={draw=none},
		height=6.5cm,
		axis on top,
		trim axis left,
		trim axis right,
		scale only axis,
		ymajorticks=false,
		xmajorticks=false,
		xmin=0,
		xmax=75,
		ymin=0,
		ymax=19
		]

		\addplot [forget plot] graphics [xmin=0,xmax=10,ymin=11,ymax=19] {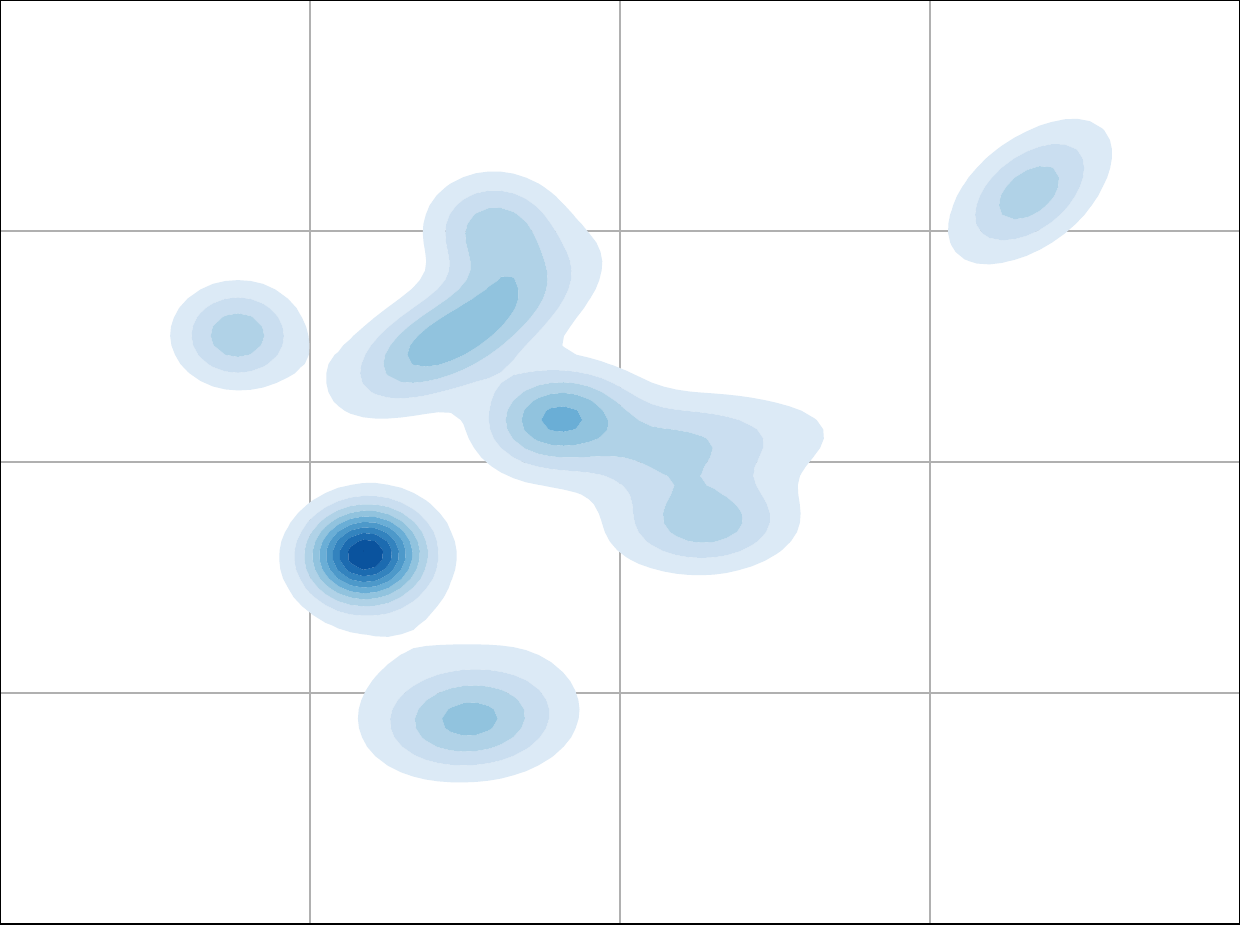};
		\addplot [forget plot] graphics [xmin=10,xmax=20,ymin=11,ymax=19] {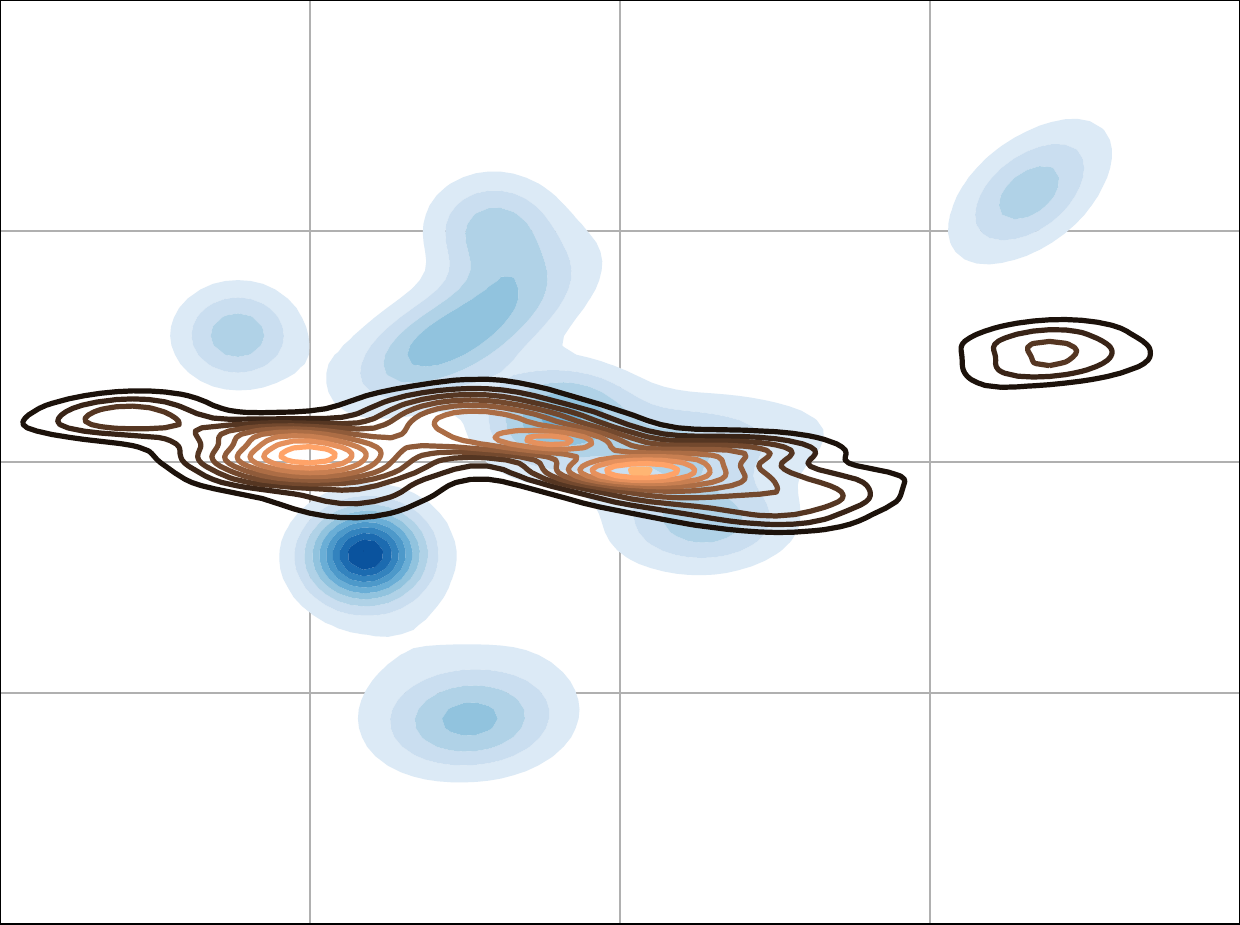};
		\addplot [forget plot] graphics [xmin=20,xmax=30,ymin=11,ymax=19] {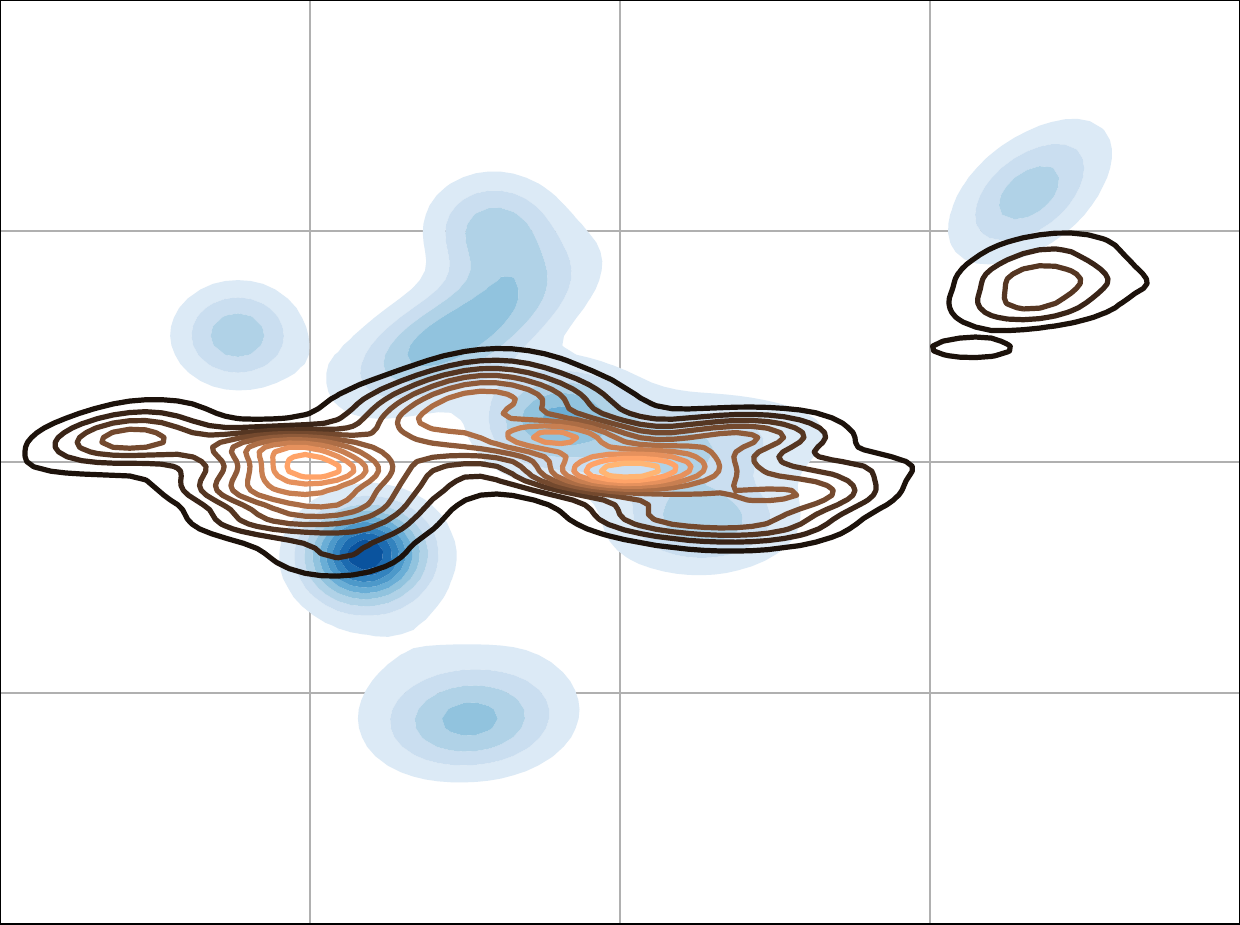};
		\addplot [forget plot] graphics [xmin=30,xmax=40,ymin=11,ymax=19] {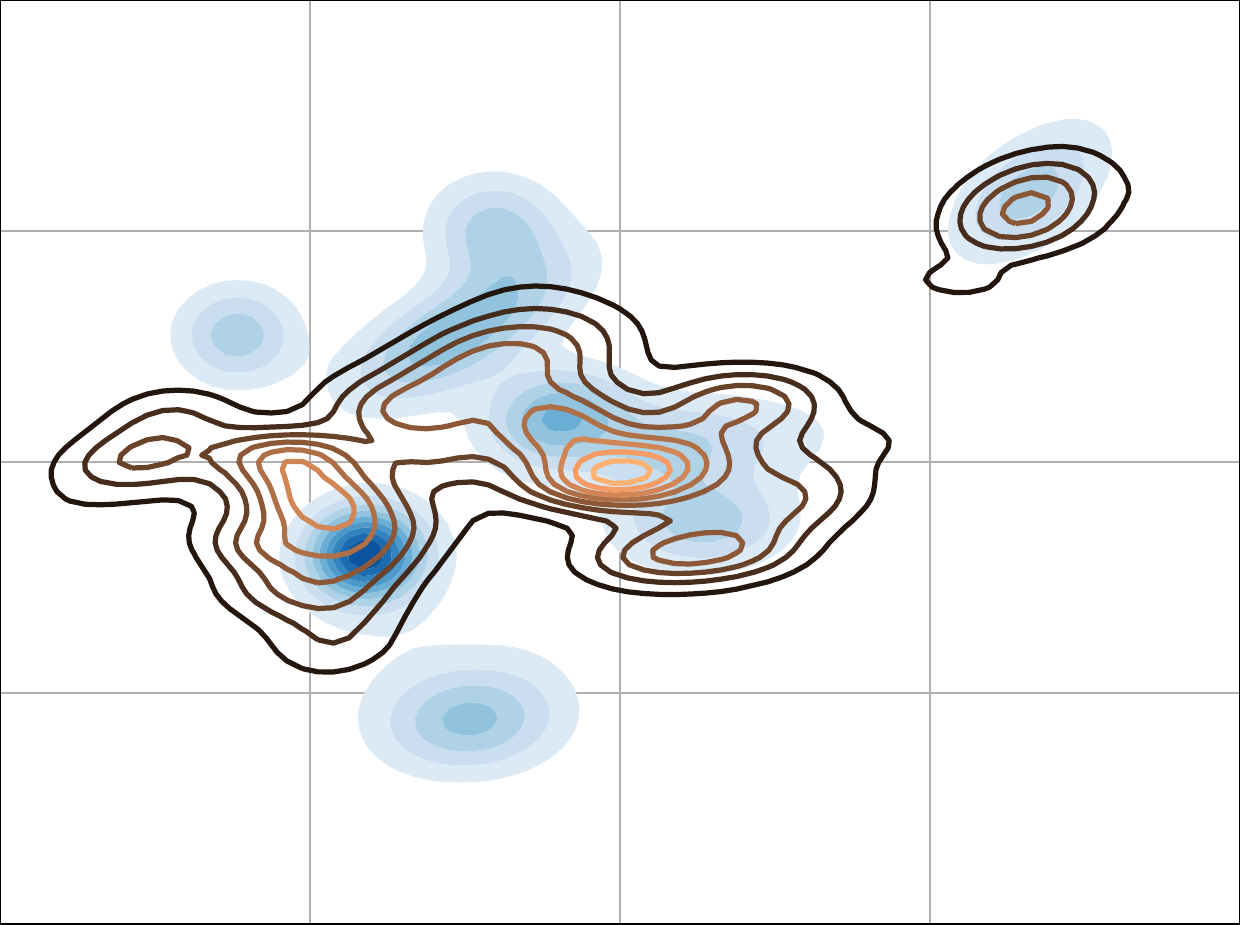};
		\addplot [forget plot] graphics [xmin=40,xmax=50,ymin=11,ymax=19] {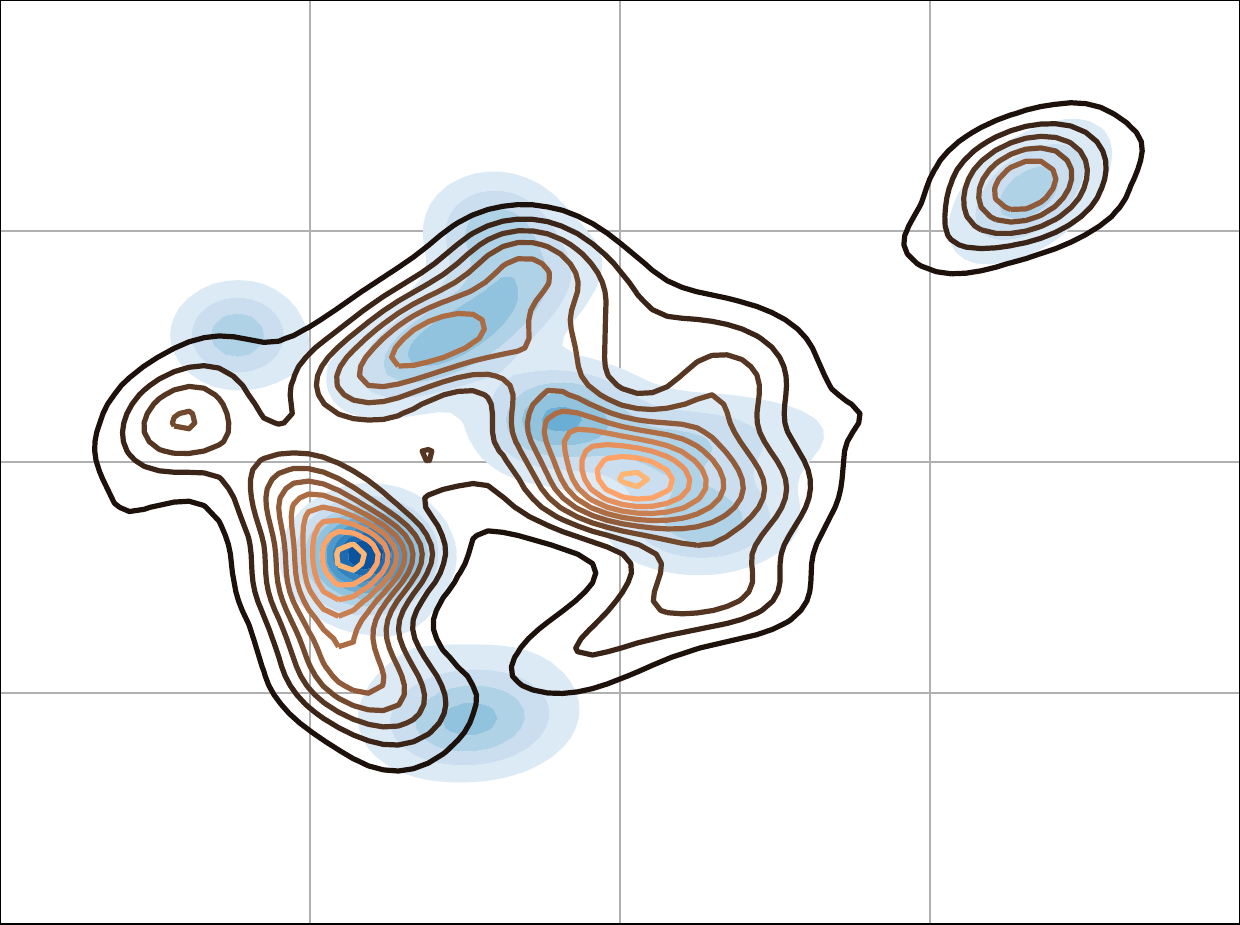};
		\node at (75,115){\small Target};
		\node at (175,115){\small$k=2$};
		\node at (275,115){\small$k=3$};
		\node at (375,115){\small$k=5$};
		\node at (475,115){\small$k=10$};

		\addplot [forget plot] graphics [xmin=0,xmax=10,ymin=2,ymax=10] {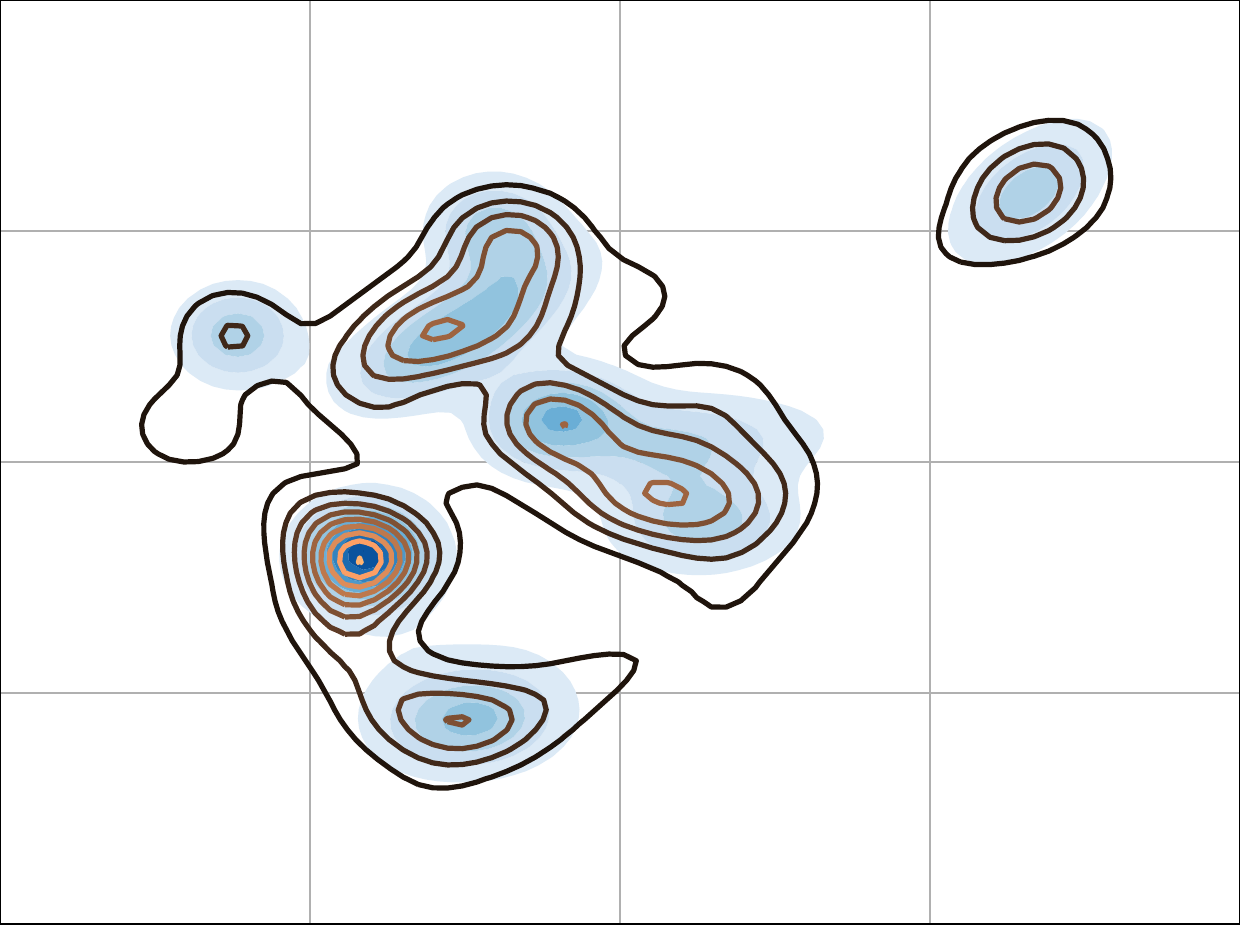};
		\addplot [forget plot] graphics [xmin=10,xmax=20,ymin=2,ymax=10] {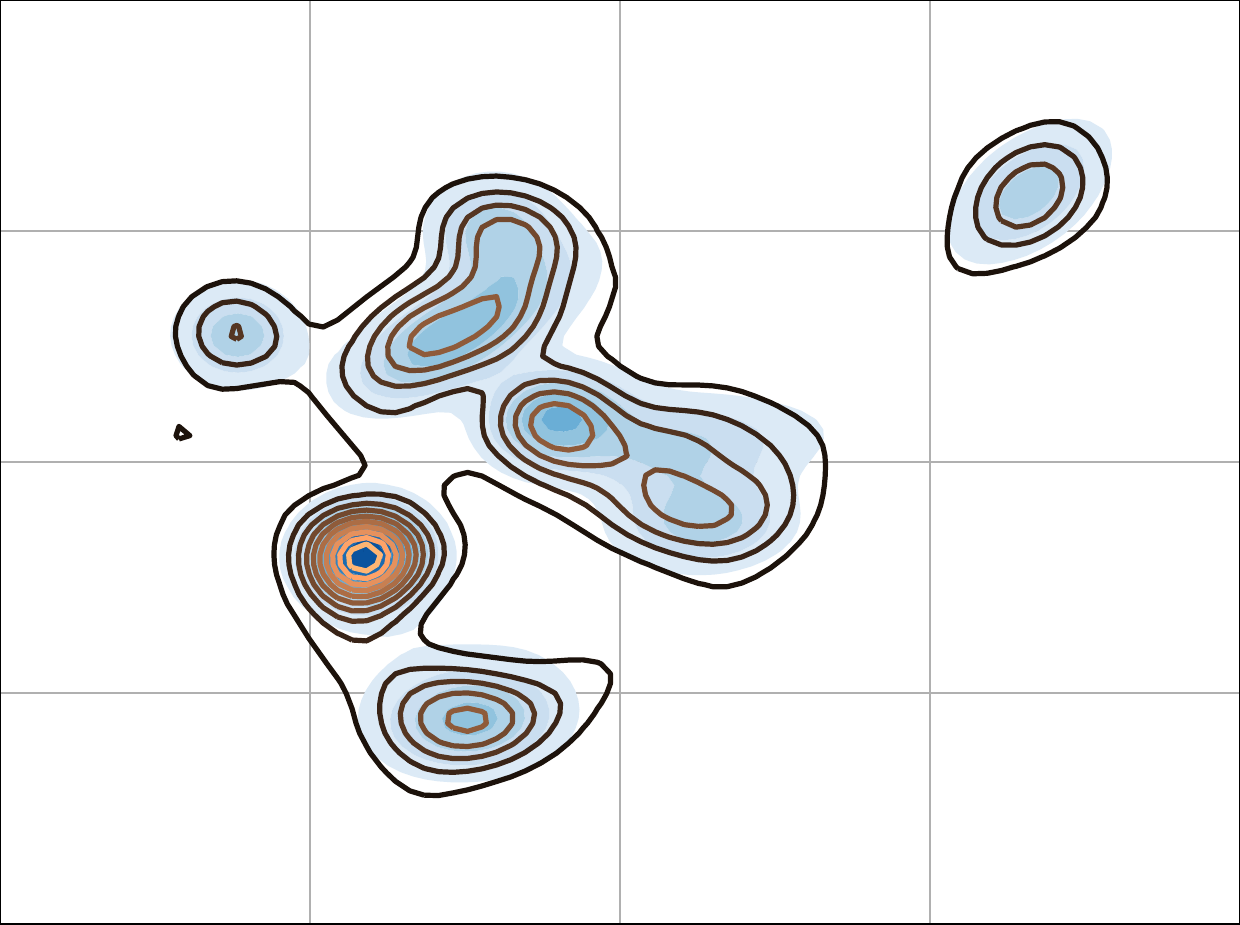};
		\addplot [forget plot] graphics [xmin=20.5,xmax=50,ymin=1,ymax=10] {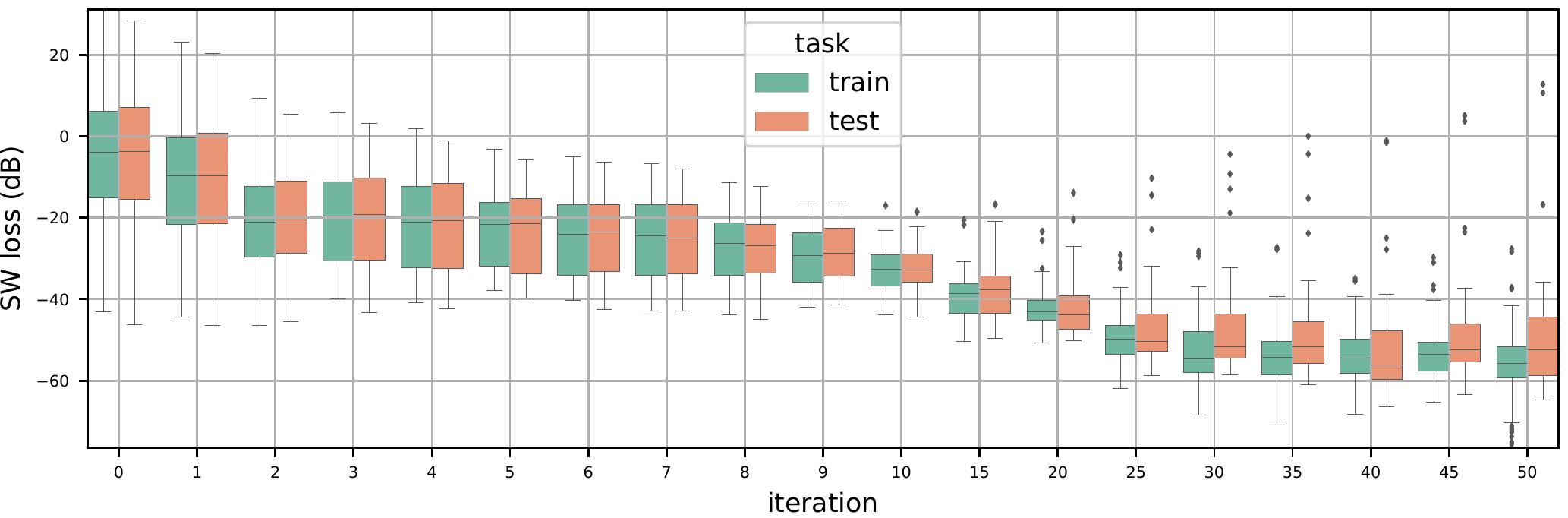};
		\node at (75,25){\small$k=20$};
		\node at (175,25){\small$k=50$};

		\draw[thick] (525,20) -- (525,250);

		\addplot [forget plot] graphics [xmin=55,xmax=65,ymin=11,ymax=19] {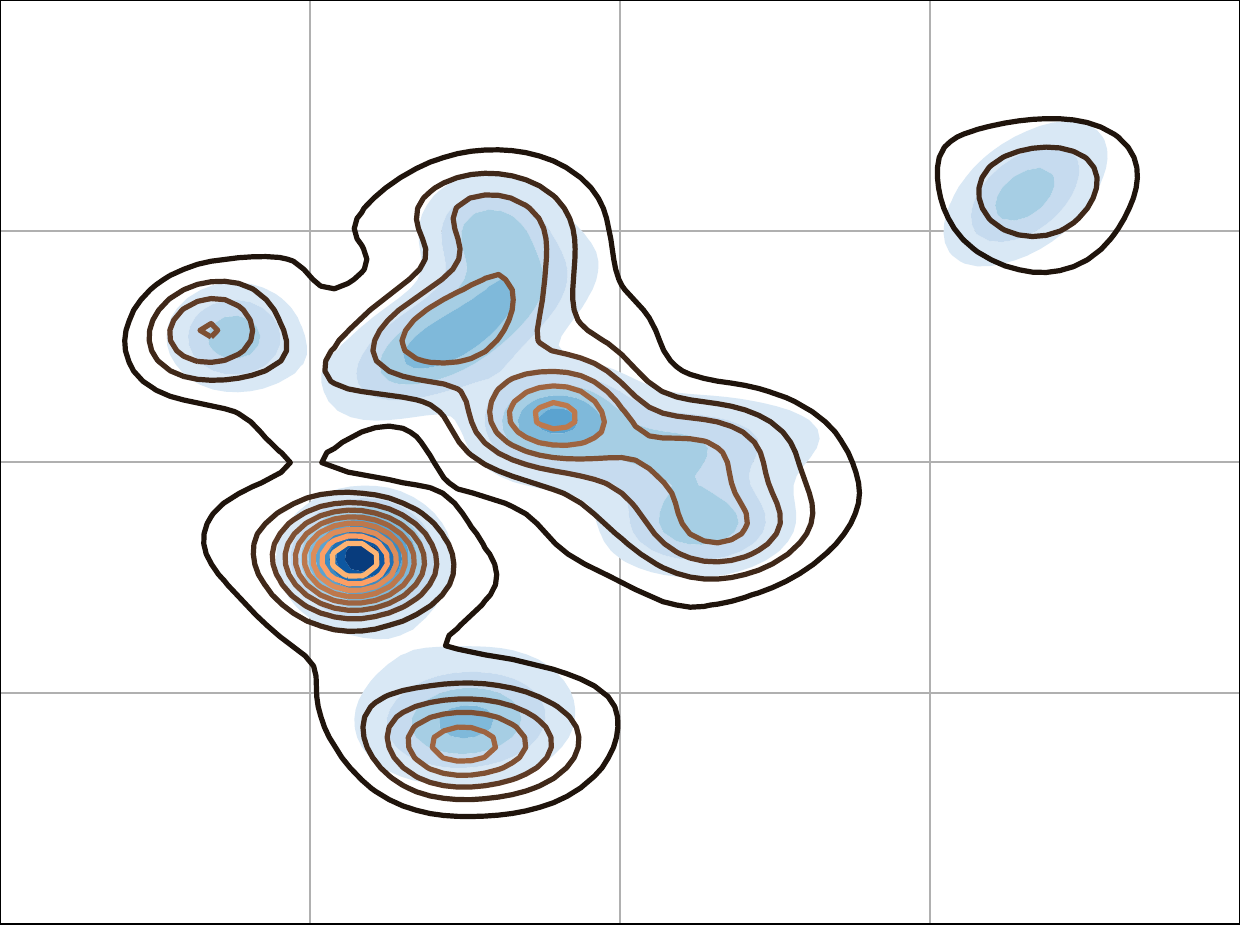};
		\addplot [forget plot] graphics [xmin=65,xmax=75,ymin=11,ymax=19] {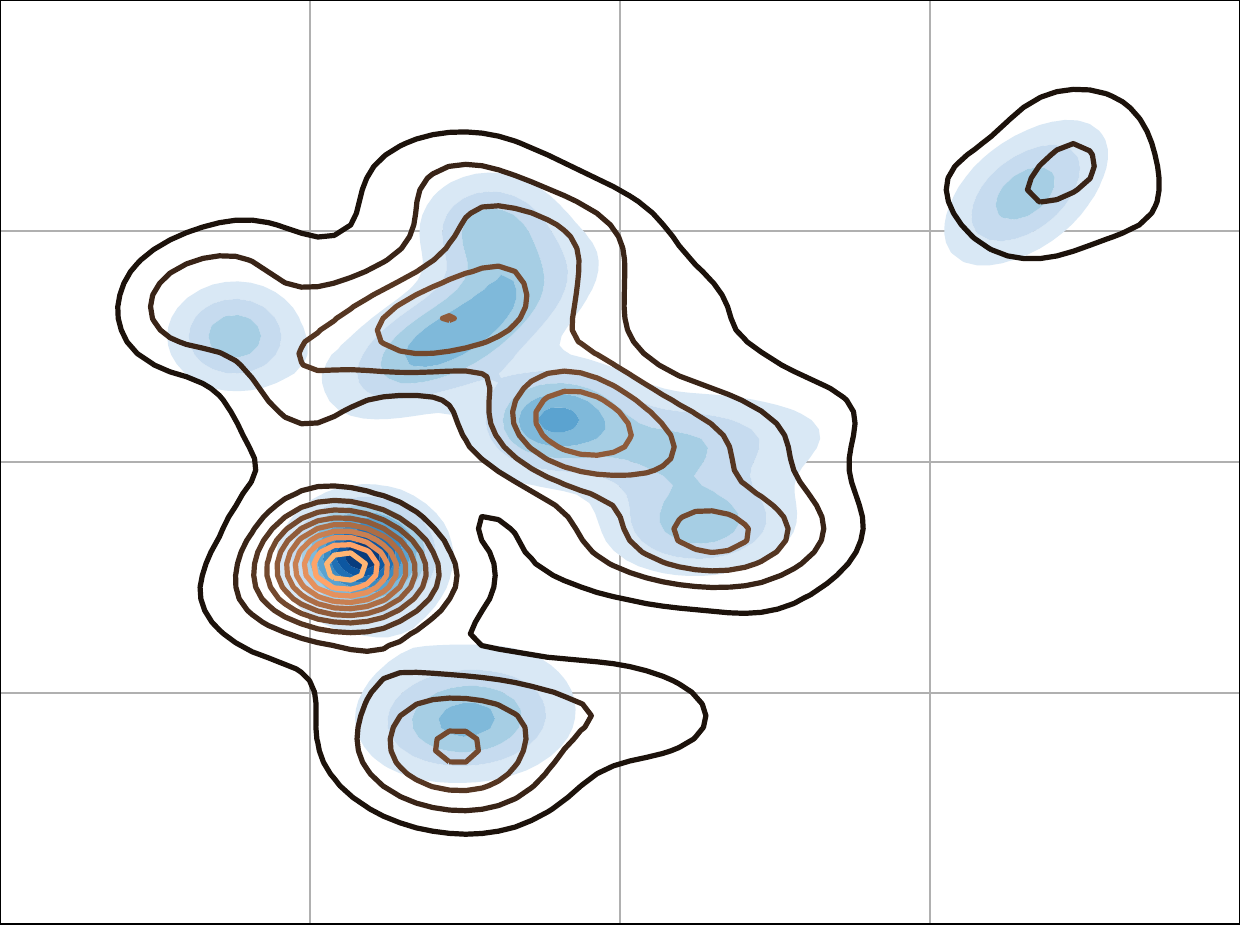};
		\node at (625,115){\small$\lambda=0.1$};
		\node at (725,115){\small$\lambda=0.2$};

		\addplot [forget plot] graphics [xmin=55,xmax=65,ymin=2,ymax=10] {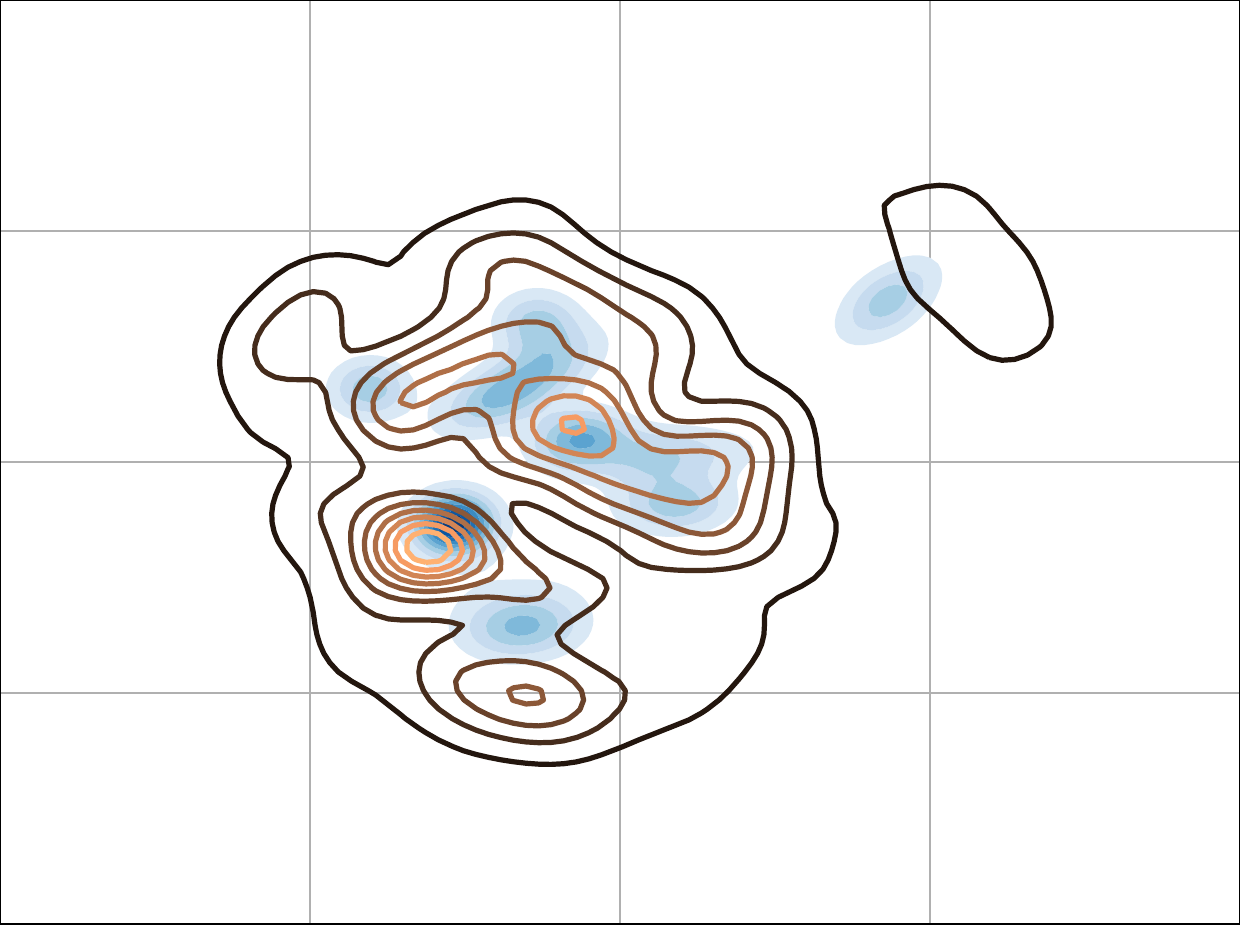};
		\addplot [forget plot] graphics [xmin=65,xmax=75,ymin=2,ymax=10] {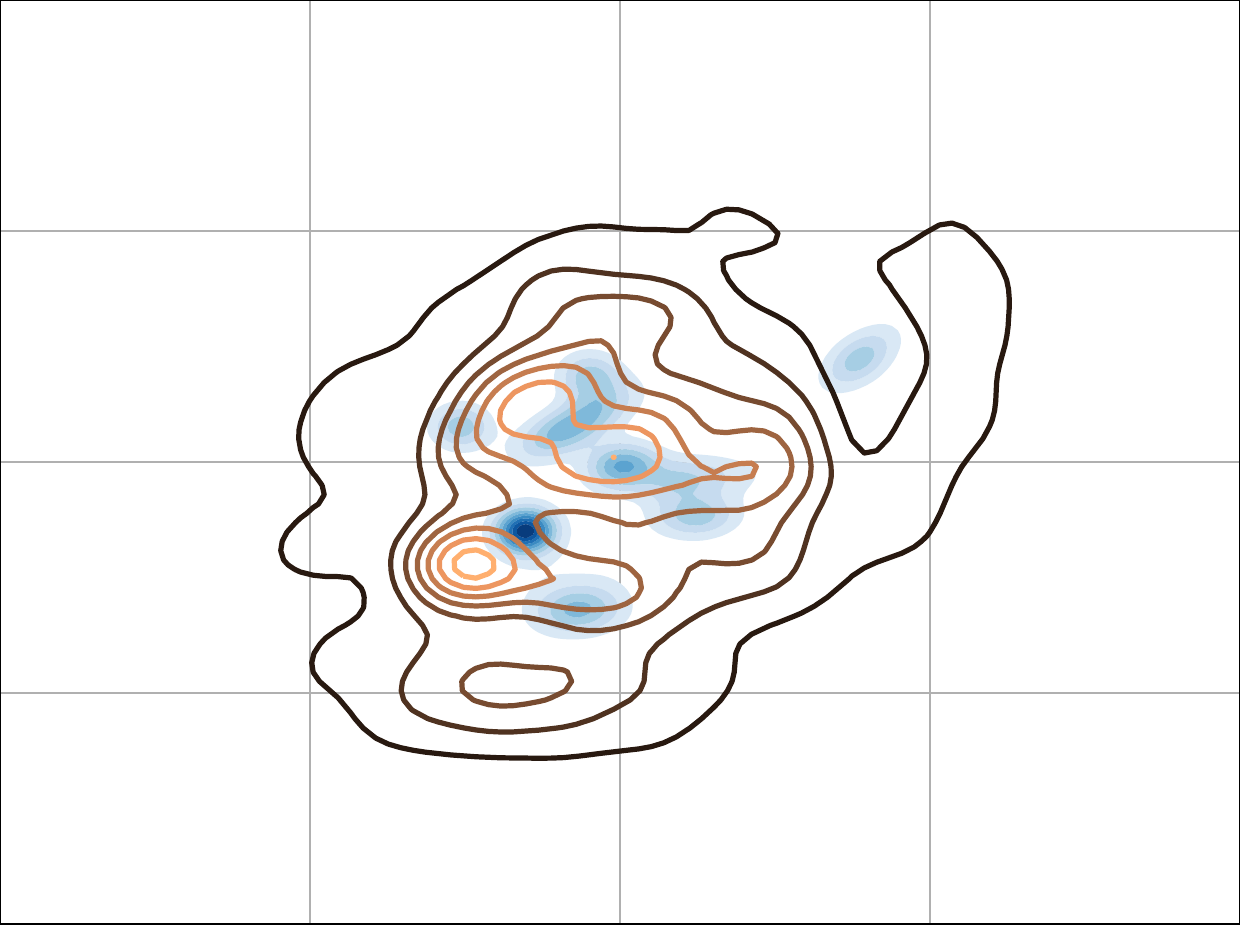};
		\node at (625,25){\small$\lambda=0.5$};
		\node at (725,25){\small$\lambda=1$};

		\end{axis}
	\end{tikzpicture}
	\vspace{-2\baselineskip}\caption{SWF on toy 2D data. \textbf{Left:} Target distribution (shaded contour plot) and distribution of particles (lines) during SWF. (bottom) SW cost over iterations during training (left) and test (right) stages. \textbf{Right:} Influence of the regularization parameter~$\lambda$. }
	\label{fig:toy_example}
\end{figure*}
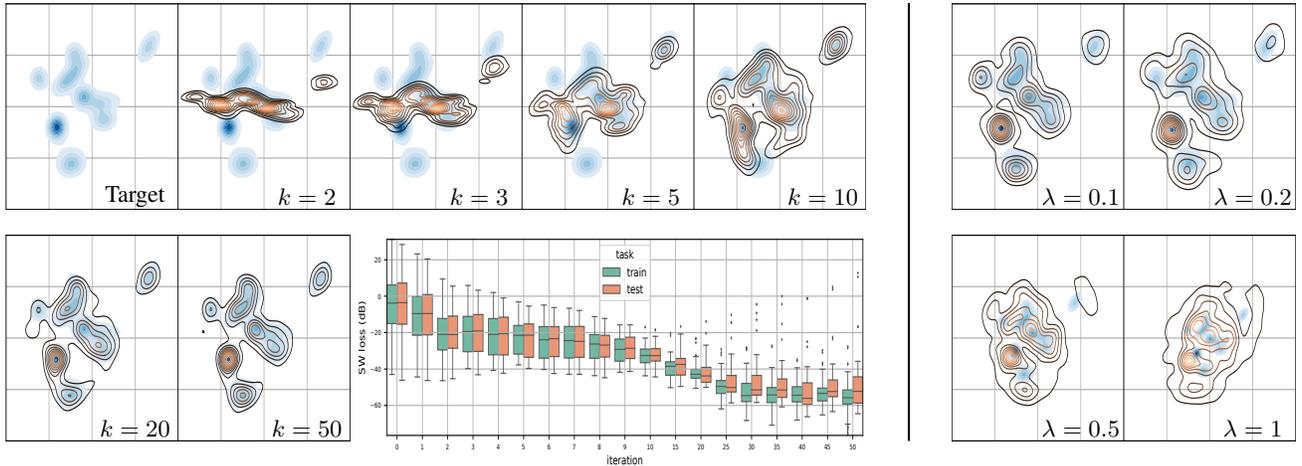

\section{Experiments}

In this section, we evaluate the SWF algorithm on a synthetic and a real data setting. Our primary goal is to validate our theory and illustrate the behavior of our non-standard approach, rather than to obtain the state-of-the-art results in IGM. In all our experiments, the initial distribution $\mu_0$ is selected as the standard Gaussian distribution on $\R^d$, we take $Q=100$ quantiles and $N=5000$ particles, which proved sufficient to approximate the quantile functions accurately.

\subsection{Gaussian Mixture Model }
We perform the first set of experiments on synthetic data where we consider a standard Gaussian mixture model (GMM) with $10$ components and random parameters. Centroids are taken as sufficiently distant from each other to make the problem more challenging. We generate $P=50000$ data samples in each experiment.

In our first experiment, we set $d=2$ for visualization purposes and illustrate the general behavior of the algorithm. Figure~\ref{fig:toy_example} shows the evolution of the particles through the iterations. Here, we set $N_\theta=30$, $h=1$ and $\lambda=10^{-4}$.
We first observe that the SW cost between the empirical distributions of training data and particles is steadily decreasing along the SW flow. Furthermore, we see that the QFs, $F^{-1}_{\theta^*_\#\bar{\mu}_{kh}^{N}}$ that are computed with the initial set of particles (the \textit{training} stage) can be perfectly re-used for new unseen particles in a subsequent \textit{test} stage, yielding similar ---~yet slightly higher~--- SW cost.

In our second experiment on Figure~\ref{fig:toy_example}, we investigate the effect of the level of the regularization $\lambda$. The distribution of the particles becomes more spread with increasing $\lambda$. This is due to the increment of the entropy, as expected.

\subsection{Experiments on real data}
\begin{figure}
\centering
\includegraphics[width=0.99\columnwidth]{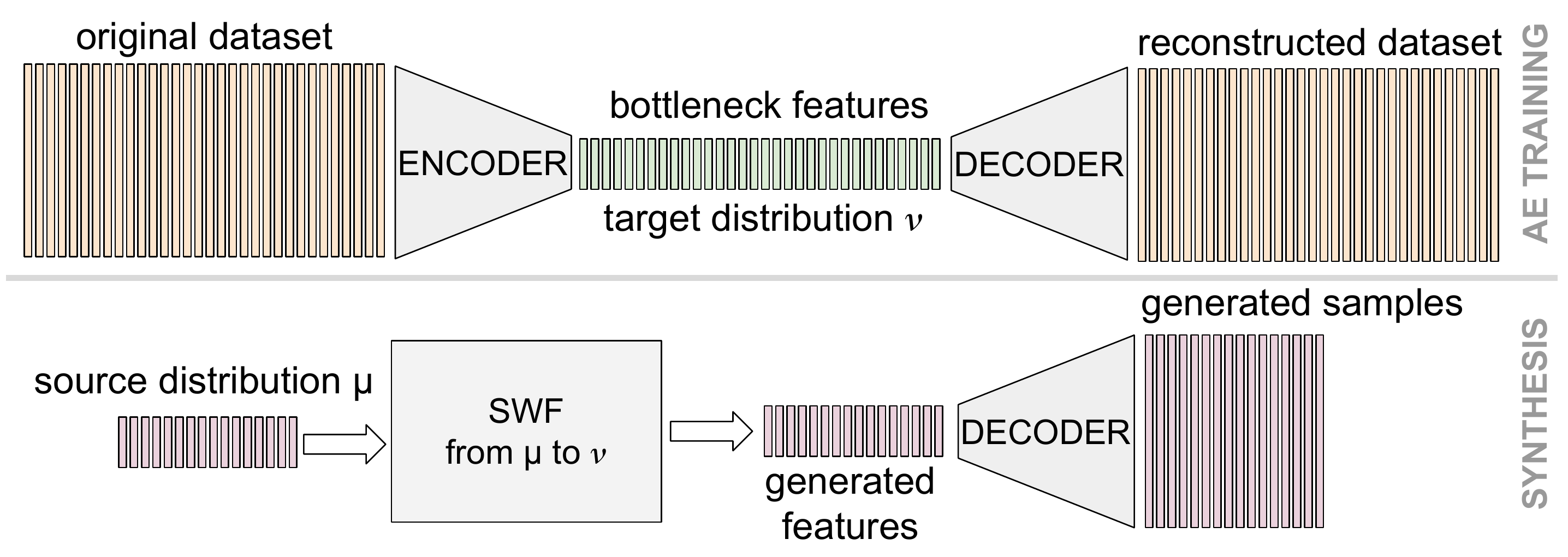}
\caption{First, we learn an autoencoder (AE). Then, we use SWF to transport random vectors to the distribution of the bottleneck features of the training set. The trained decoder is used for visualization.}
\label{fig:using_ae}
\vspace{-10pt}
\end{figure}

\begin{figure}[h]
\centering
\includegraphics[width=0.49\columnwidth]{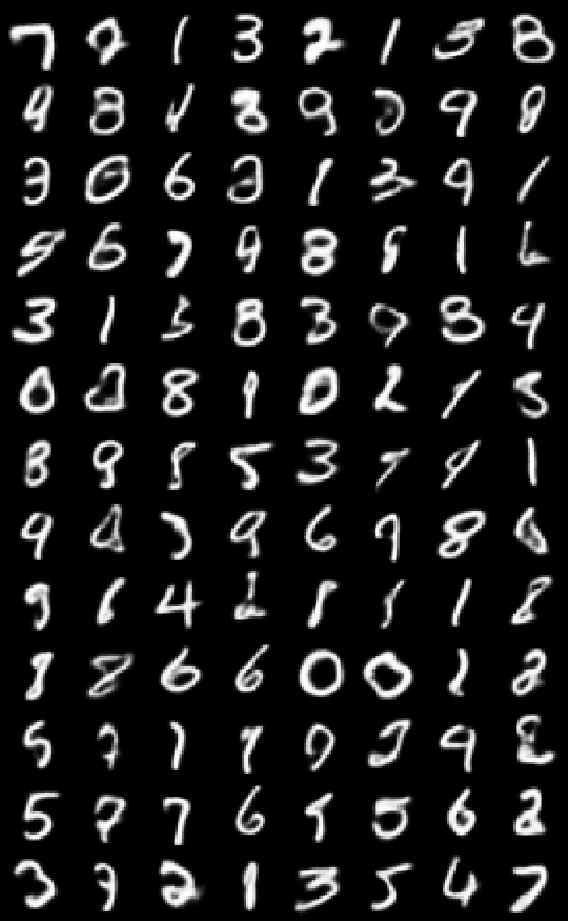}
\includegraphics[width=0.49\columnwidth]{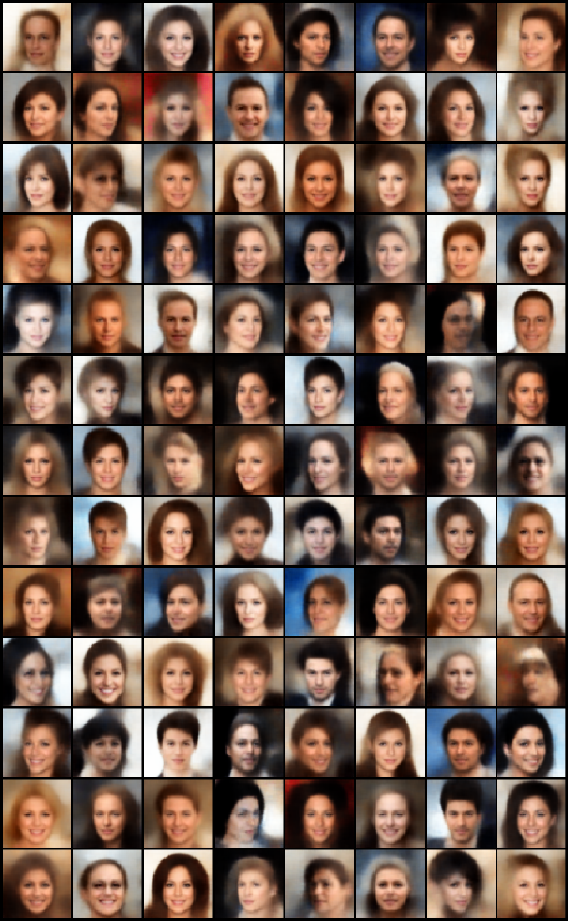}
\vspace{-\baselineskip}
\caption{Samples generated after 200 iterations of SWF to match the distribution of bottleneck features for the training dataset. Visualization is done with the pre-trained decoder.\label{fig:samples}}
\vspace{-\baselineskip}
\end{figure}

In the second set of experiments, we test the SWF algorithm on two real datasets. (i) The traditional MNIST dataset that contains 70K binary images corresponding to different digits. (ii) The popular CelebA dataset \cite{liu2015faceattributes}, that contains $202$K color-scale images. This dataset is advocated as more challenging than MNIST. Images were interpolated as $32\times 32$ for MNIST, and $64\times 64$ for CelebA.

In experiments reported in \supp, we found out that directly applying SWF to such high-dimensional data yielded noisy results, possibly due to the insufficient sampling of $\Sp^{d-1}$. To reduce the dimensionality, we trained a standard convolutional autoencoder (AE) on the training set of both datasets (see Figure~\ref{fig:using_ae} and \supp), and the target distribution $\nu$ considered becomes the distribution of the resulting bottleneck features,
with dimension $d$. Particles can be visualized with the pre-trained decoder.
Our goal is to show that SWF permits to directly sample from the distribution of bottleneck features, as an alternative to enforcing this distribution to match some prior, as in VAE. In the following, we set $\lambda=0$, $N_\theta=40000$, $d=32$ for MNIST and $d=64$ for CelebA.

Assessing the validity of IGM algorithms is generally done by visualizing the generated samples. Figure~\ref{fig:samples} shows some particles after $500$ iterations of SWF. We can observe they are considerably accurate. Interestingly, the generated samples gradually take the form of either digits or faces along the iterations, as seen on Figure~\ref{fig:evolution}. In this figure, we also display the closest sample from the original database to check we are not just reproducing training data.

\begin{figure}[]
\centering
\begin{tikzpicture}[font=\small]

	\begin{axis}[%
		/pgf/number format/.cd,
				use comma,
				1000 sep={},
	width=1\columnwidth,
	axis line style={draw=none},
	height=10cm,
	axis on top,
	trim axis left,
	scale only axis,
	ymajorticks=false,
	xtick={4,12,20,28,36,44,52,60,68,76},
	xticklabels={3,5,8,10,15,20,30,50,100,200},
	xtick align=center,
	ticklabel style = {font=\tiny},
	xmin=-10,
	xmax=90,
	ymin=0,
	ymax=16
	]

	\addplot [forget plot] graphics [xmin=-10,xmax=-2,ymin=8,ymax=16] {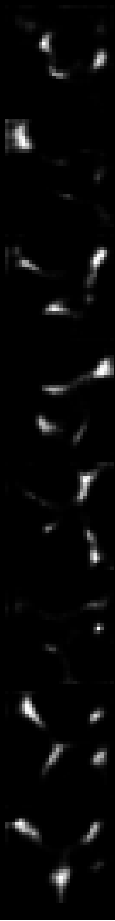};
	\addplot [forget plot] graphics [xmin=0.0,xmax=8,ymin=8,ymax=16] {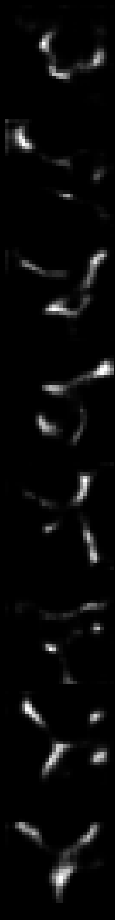};
	\addplot [forget plot] graphics [xmin=7.5,xmax=16,ymin=8,ymax=16] {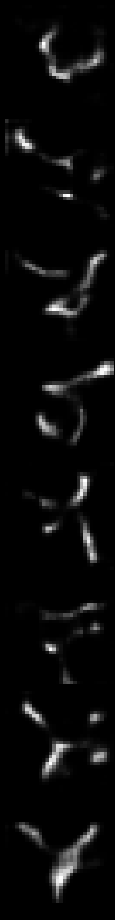};
	\addplot [forget plot] graphics [xmin=15.5,xmax=24,ymin=8,ymax=16] {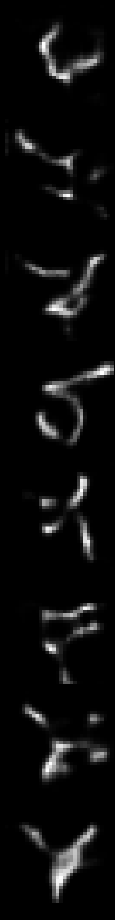};
	\addplot [forget plot] graphics [xmin=23.5,xmax=32,ymin=8,ymax=16] {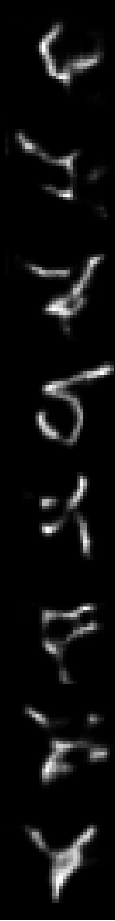};
	\addplot [forget plot] graphics [xmin=31.5,xmax=40,ymin=8,ymax=16] {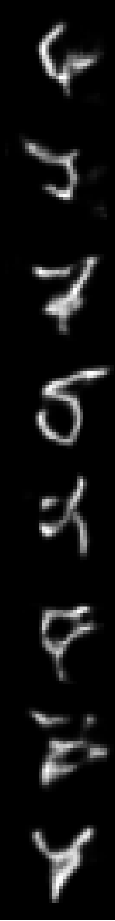};
	\addplot [forget plot] graphics [xmin=39.5,xmax=48,ymin=8,ymax=16] {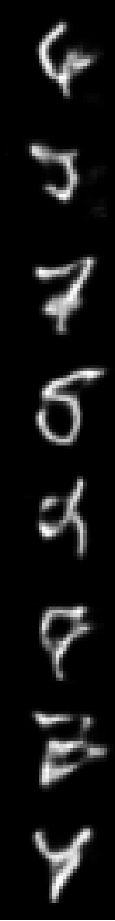};
	\addplot [forget plot] graphics [xmin=47.5,xmax=56,ymin=8,ymax=16] {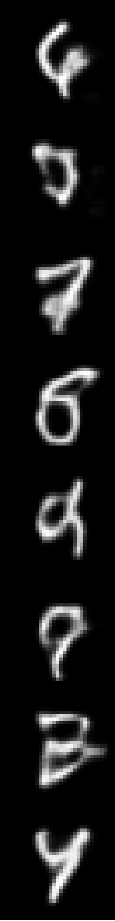};
	\addplot [forget plot] graphics [xmin=55.5,xmax=64,ymin=8,ymax=16] {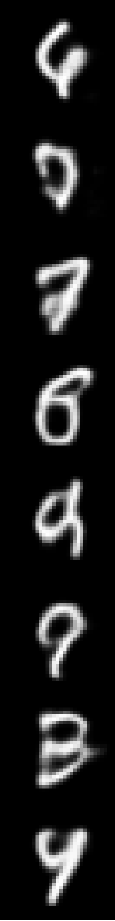};
	\addplot [forget plot] graphics [xmin=63.5,xmax=72,ymin=8,ymax=16] {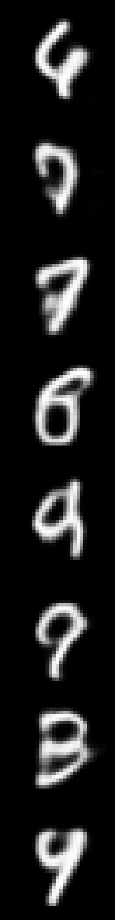};
	\addplot [forget plot] graphics [xmin=71.5,xmax=80,ymin=8,ymax=16] {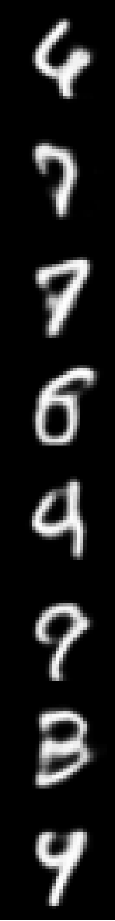};
	\addplot [forget plot] graphics [xmin=82.0,xmax=90,ymin=8,ymax=16] {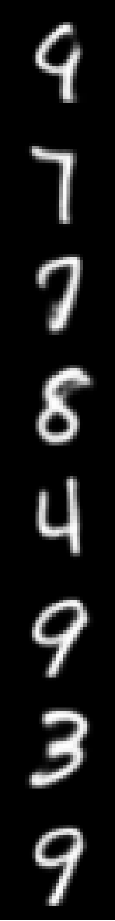};

	\addplot [forget plot] graphics [xmin=-10,xmax=-2,ymin=0,ymax=8] {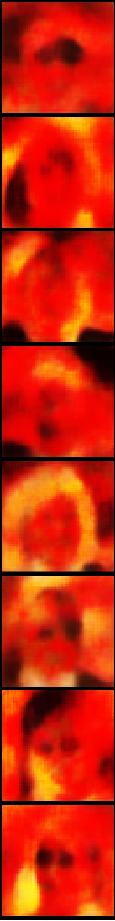};
	\addplot [forget plot] graphics [xmin=0.0,xmax=8,ymin=0,ymax=8] {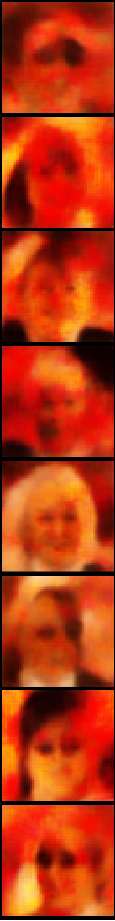};
	\addplot [forget plot] graphics [xmin=7.5,xmax=16,ymin=0,ymax=8] {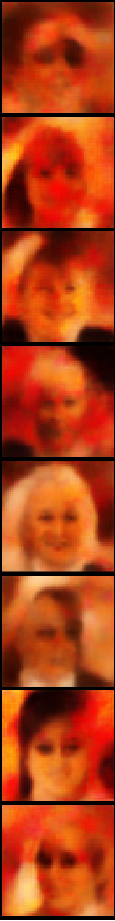};
	\addplot [forget plot] graphics [xmin=15.5,xmax=24,ymin=0,ymax=8] {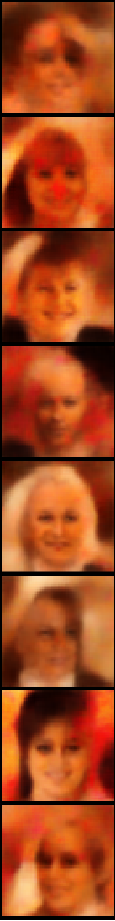};
	\addplot [forget plot] graphics [xmin=23.5,xmax=32,ymin=0,ymax=8] {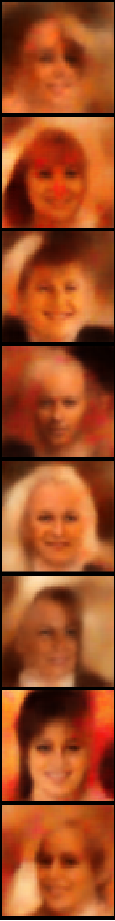};
	\addplot [forget plot] graphics [xmin=31.5,xmax=40,ymin=0,ymax=8] {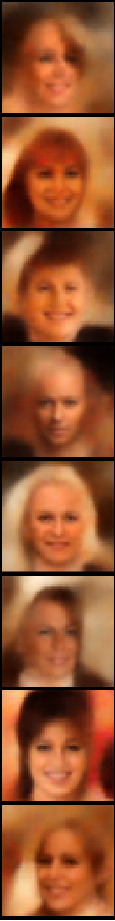};
	\addplot [forget plot] graphics [xmin=39.5,xmax=48,ymin=0,ymax=8] {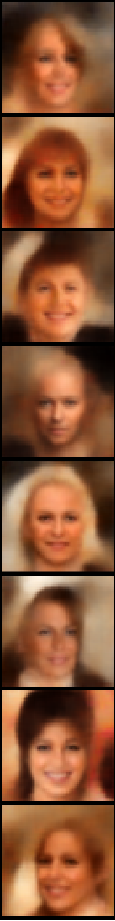};
	\addplot [forget plot] graphics [xmin=47.5,xmax=56,ymin=0,ymax=8] {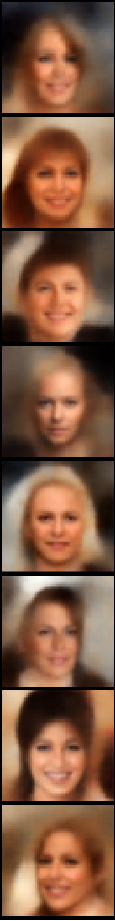};
	\addplot [forget plot] graphics [xmin=55.5,xmax=64,ymin=0,ymax=8] {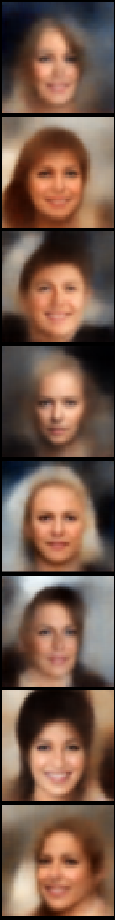};
	\addplot [forget plot] graphics [xmin=63.5,xmax=72,ymin=0,ymax=8] {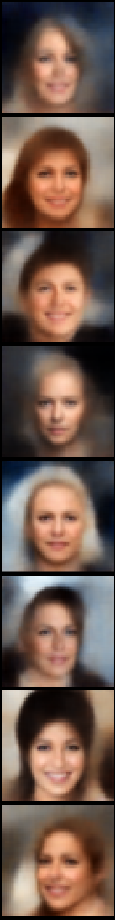};
	\addplot [forget plot] graphics [xmin=71.5,xmax=80,ymin=0,ymax=8] {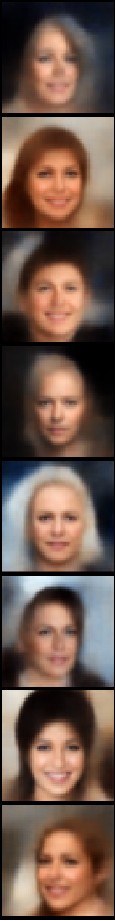};
	\addplot [forget plot] graphics [xmin=82.0,xmax=90,ymin=0,ymax=8] {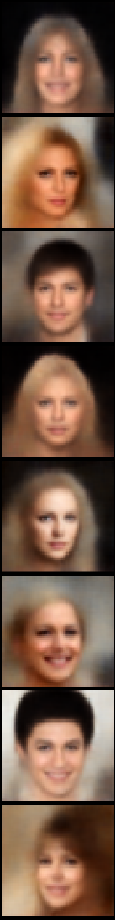};

	\end{axis}
\end{tikzpicture}
\vspace{-20pt}
\caption{Initial random particles (left), particles through iterations (middle, from 1 to 200 iterations) and closest sample from the training dataset (right), for both MNIST and CelebA.\label{fig:evolution}}
\end{figure}

\begin{figure}[]
\centering
\includegraphics[width=0.99\columnwidth]{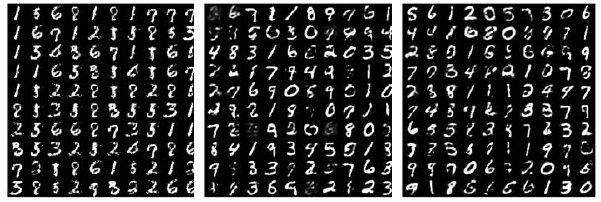}
\vspace{-\baselineskip}
\caption{Performance of GAN (left), W-GAN (middle), SWG (right) on MNIST. (The figure is directly taken from \cite{deshpande2018generative}.) }
\label{fig:mnistall}
\vspace{-10pt}
\end{figure}

For a visual comparison, we provide the results presented in \cite{deshpande2018generative} in Figure~\ref{fig:mnistall}. These results are obtained by running different IGM approaches on the MNIST dataset, namely GAN \cite{goodfellow2014generative}, Wasserstein GAN (W-GAN) \cite{arjovsky2017wasserstein} and the Sliced-Wasserstein Generator (SWG) \cite{deshpande2018generative}.
The visual comparison suggests that the samples generated by SWF are of slightly better quality than those, although research must still be undertaken to scale up to high dimensions without an AE.

\begin{figure}
\centering
\includegraphics[width=0.95\columnwidth]{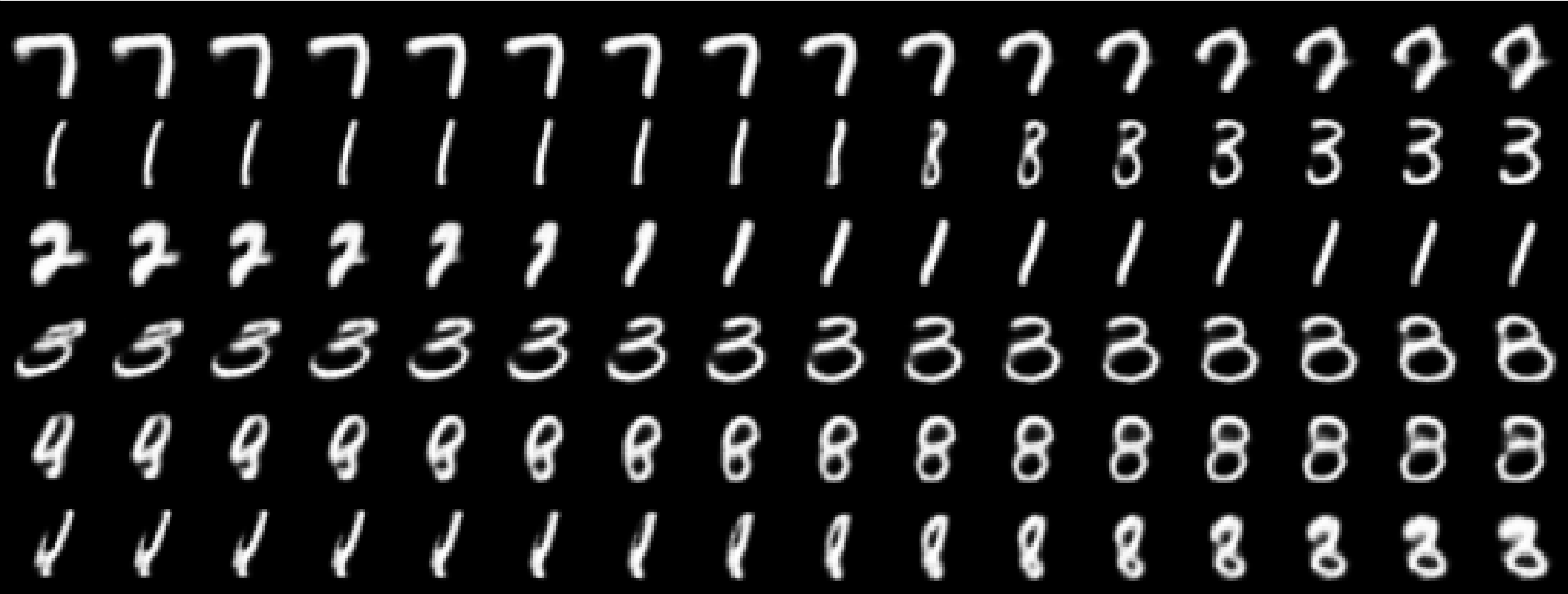}
\includegraphics[width=0.95\columnwidth]{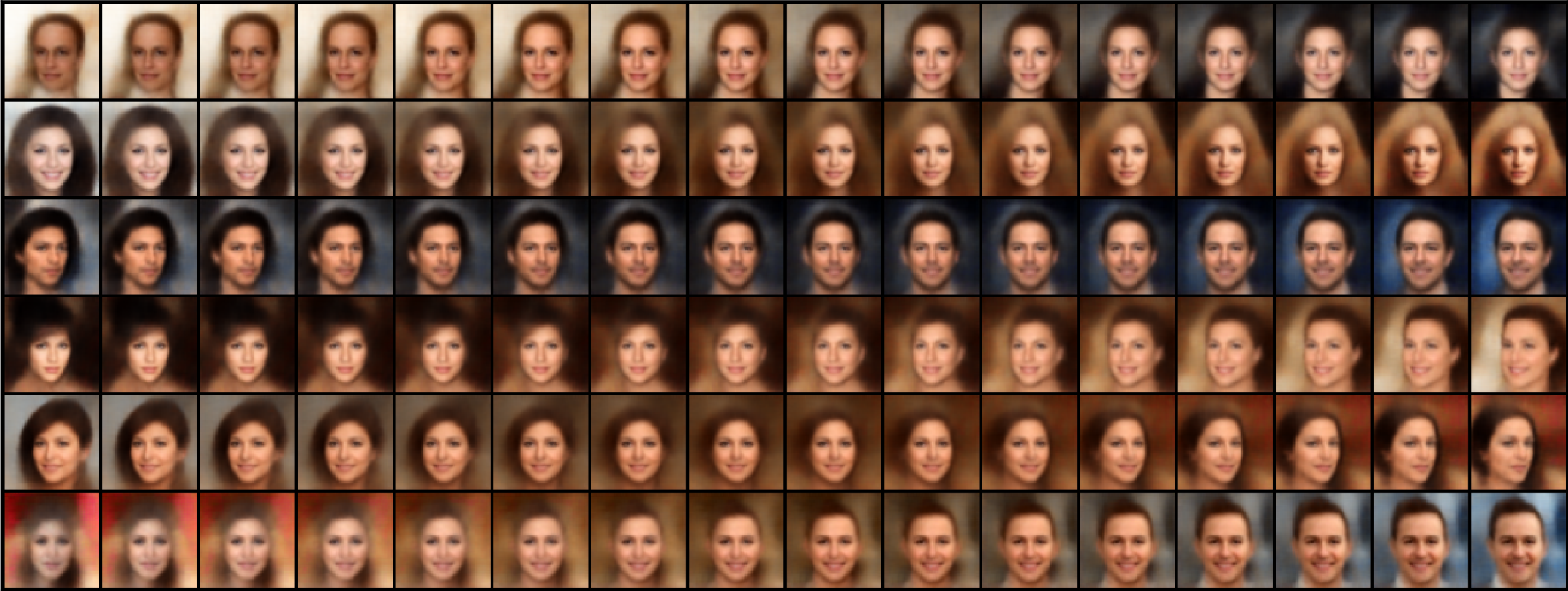}
\label{fig:interpolation}
\vspace{-\baselineskip}
\caption{Applying a pre-trained SWF on new samples located in-between the ones used for training. Visualization is done with the pre-trained decoder.}
\vspace{-10pt}
\end{figure}

We also provide the outcome of the pre-trained SWF with samples that are regularly spaced in between those used for training. The result is shown in Figure \ref{fig:interpolation}. This plot suggests that SWF is a way to interpolate non-parametrically in between latent spaces of regular AE.

\section{Conclusion and Future Directions}

In this study, we proposed SWF, an efficient, nonparametric IGM algorithm.
SWF is based on formulating IGM as a functional optimization problem in Wasserstein spaces, where the aim is to find a probability measure that is close to the data distribution as much as possible while maintaining the expressiveness at a certain level.
SWF lies in the intersection of OT, gradient flows, and SDEs, which allowed us to convert the IGM problem to an SDE simulation problem. We provided finite-time bounds for the infinite-particle regime and established explicit links between the algorithm parameters and the overall error.
We conducted several experiments, where we showed that the results support our theory: SWF is able to generate samples from non-trivial distributions with low computational requirements.

The SWF algorithm opens up interesting future directions: (i) extension to differentially private settings \cite{dwork2014algorithmic} by exploiting the fact that it only requires random projections of the data, (ii) showing the convergence scheme of the particle system \eqref{eqn:sde_particle} to the original SDE \eqref{eqn:sde}, (iii) providing bounds directly for the particle scheme \eqref{eqn:euler_particle}.

\section*{Acknowledgments}
This work is partly supported by the French National Research Agency (ANR) as a part of the FBIMATRIX
(ANR-16-CE23-0014) and KAMoulox (ANR-15-CE38-0003-01) projects. Szymon Majewski is partially supported by Polish National Science Center grant number 2016/23/B/ST1/00454.

\bibliography{references}
\bibliographystyle{icml2019}

\vfill
\pagebreak

\newpage

\onecolumn

\icmltitle{Sliced-Wasserstein Flows: Nonparametric Generative Modeling via Optimal Transport and Diffusions \\ {\large SUPPLEMENTARY DOCUMENT}}

\icmlsetsymbol{equal}{*}

\begin{icmlauthorlist}
\icmlauthor{Antoine Liutkus}{inria}
\icmlauthor{Umut \c{S}im\c{s}ekli}{telecool}
\icmlauthor{Szymon Majewski}{impan}
\icmlauthor{Alain Durmus}{ens}
\icmlauthor{Fabian-Robert St\"oter}{inria}
\end{icmlauthorlist}

\icmlaffiliation{inria}{Inria and LIRMM, Univ. of Montpellier, France}
\icmlaffiliation{telecool}{LTCI, T\'{e}l\'{e}com Paristech, Universit\'{e} Paris-Saclay, Paris, France }
\icmlaffiliation{impan}{Institute of Mathematics, Polish Academy of Sciences, Warsaw, Poland}
\icmlaffiliation{ens}{CNRS, ENS Paris-Saclay,Université Paris-Saclay, Cachan, France}

\icmlcorrespondingauthor{Antoine Liutkus}{antoine.liutkus@inria.fr}
\icmlcorrespondingauthor{Umut Simsekli}{umut.simsekli@telecom-paristech.fr}

\icmlkeywords{Machine Learning, ICML}

\setcounter{section}{0}
\setcounter{equation}{0}
\setcounter{figure}{0}
\setcounter{table}{0}
\setcounter{page}{1}
\renewcommand{\theequation}{S\arabic{equation}}
\renewcommand{\thefigure}{S\arabic{figure}}
\renewcommand{\thelemma}{S\arabic{lemma}}
\renewcommand{\theassumption}{S\arabic{assumption}}
\renewcommand{\thethm}{S\arabic{thm}}
\renewcommand{\theprop}{S\arabic{prop}}

\section{Proof of Theorem~\ref{thm:continuity}}

We first need to generalize \cite{bonnotte2013unidimensional}[Lemma 5.4.3] to distribution $\rho \in \mrl^{\infty}(\cB(0,r))$, $r >0$.
\begin{thm} \label{thm:implicit_step}
 Let $\nu$ be a probability measure on $\cB(0,1)$ with a strictly positive smooth density. Fix a time step $h > 0$, regularization constant $\lambda > 0$ and a radius $r > \sqrt{d}$. For any probability measure $\mu_0$ on $\cB(0, r)$ with density $\rho_0 \in \mrl^{\infty}(\cB(0, r))$, there is a probability measure $\mu$ on $\cB(0,r)$ minimizing:
\[
\mathcal{G}(\mu) = \mcf^{\nu}_{\lambda} (\mu) + \frac{1}{2h} \W^2(\mu, \mu_0) ,
\]
where $\mcf^{\nu}_\lambda$ is given by \eqref{eqn:sw_optim}.
Moreover the optimal $\mu$ has a density $\rho$ on $\cB(0,r)$ and:
\begin{equation} \label{ineq:inf_norm_bound}
||\rho||_{\mrl^{\infty}} \leq (1 + h/\sqrt{d})^d ||\rho_0||_{\mrl^{\infty}}.
\end{equation} 
\end{thm}
\begin{proof}
The set of measures supported on $\cB(0,r)$ is compact in the topology given by $\W$ metric. Furthermore by \cite{ambrosio2008gradient}[Lemma 9.4.3] $\He$ is lower semicontinuous on $(\PS(\cB(0,r)),\W)$. Since by \cite{bonnotte2013unidimensional}[Proposition 5.1.2, Proposition 5.1.3], $\SW$ is a distance  on $\PS(\cB(0,r))$, dominated by $d^{-1/2} \W$, we have:
\[
|\SW(\pi_0, \nu) - \SW(\pi_1, \nu)| \leq \SW(\pi_0, \pi_1) \leq \frac{1}{\sqrt{d}}\W(\pi_0, \pi_1).
\]
The above means that $\SW(\cdot, \nu)$ is continuous with respect to topology given by $\W$, which implies that $\SW^2(\cdot, \nu)$ is continuous in this topology as well. Therefore $\mcg : \PS(\cB(0,r)) \to \ocint{-\infty,\plusinfty}$ is a lower semicontinuous function on the compact set $(\PS(\cB(0,r)),\W)$. Hence there exists a minimum  $\mu$ of $\mcg$ on $\mathcal{P}(\cB(0,r))$. Furthermore, since $\mch(\pi) = +\infty$  for measures $\pi$ that do not admit a density with respect to Lebesgue measure, the measure $\mu$ must admit a density $\rho$.

If $\rho_0$ is smooth and positive on $\cB(0,r)$, the inequality \ref{ineq:inf_norm_bound} is true by \cite{bonnotte2013unidimensional}[Lemma 5.4.3.] When $\rho_0$ is just in $\mrl^{\infty}(\cB(0,r))$, we proceed by smoothing. 
For $t \in (0,1]$, let $\rho_t$ be a function obtained by convolution of $\rho_0$ with a Gaussian kernel $(t,x,y) \mapsto (2\pi)^{d/2} \exp(\norm[2]{x-y}/2)$, restricting the result to $\cB(0,r)$ and normalizing to obtain a probability density. Then $(\rho_t)_{t}$ are smooth positive densities, and it is easy to see that $\lim_{t \rightarrow 0} ||\rho_t||_{\mrl^{\infty}} \leq ||\rho_0||_{\mrl^{\infty}}$. Furthermore, if we denote by $\mu_t$ the measure on $\cB(0,r)$ with density $\rho_t$, then $\mu_t$ converge weakly to $\mu_0$.
For $t \in (0, 1]$ let $\hat{\mu}_t$ be the minimum of $ \mcf^{\nu}_{\lambda}(\cdot) + \frac{1}{2h} \W^2(\cdot, \mu_t)$, and let $\hat{\rho}_t$ be the density of $\hat{\mu}_t$. Using \cite{bonnotte2013unidimensional}[Lemma 5.4.3.] we get 
\[
||\hat{\rho}_t ||_{\mrl^{\infty}} \leq (1 + h\sqrt{d})^d ||\rho_t||_{\mrl^{\infty}}.
\]
so $\hat{\rho_{t}}$ lies in a ball of finite radius in $\mrl^{\infty}$.  Using compactness of $\mathcal{P}(\cB(0,r))$ in weak topology and compactness of closed ball in $\mrl^{\infty}(\cB(0,r))$ in weak star topology, we can choose a subsequence $\hat{\mu}_{t_k} , \hat{\rho}_{t_k}$, $\lim_{k \to \plusinfty} t_k =0$, that converges along that subsequence to limits $\hat{\mu}$, $\hat{\rho}$. Obviously $\hat{\rho}$ is the density of $\hat{\mu}$, since for any continuous function $f$  on $\cB(0,r)$ we have:
\[
\int \hat{\rho} f dx = \lim_{k \rightarrow \infty} \int \rho_{t_k} f dx = \lim_{k \rightarrow \infty} \int f d\mu_{t_k} = \int f d\mu.
\]
Furthermore, since $\hat{\rho}$ is the weak star limit of a bounded subsequence, we have:
\[
||\hat{\rho} ||_{\mrl^{\infty}} \leq \limsup_{k \rightarrow \infty}(1 + h\sqrt{d})^d ||\rho_{t_k}||_{\mrl^{\infty}} \leq (1 + h\sqrt{d})^d ||\rho_0||_{\mrl^{\infty}}.
\]
To finish, we just need to prove that $\hat{\mu}$ is a minimum of $\mcg$. We remind our reader, that we already established existence of some minimum $\mu$ (that might be different from $\hat{\mu}$). Since $\hat{\mu}_{t_k}$ converges weakly to $\hat{\mu}$ in $\mathcal{P}(\cB(0,r))$, it implies convergence  in $\W$ as well since $\cB(0,r)$ is compact. Similarly $\mu_{t_k}$ converges to $\mu_0$ in $\W$. Using the lower semicontinuity of $\mcg$ we now have:
\[
\begin{aligned}
\mcf^{\nu}_{\lambda}(\hat{\mu}) + \frac{1}{2h} \W^2(\hat{\mu}, \mu_0)  & \leq \liminf_{k \rightarrow \infty} \left( \mcf^{\nu}_{\lambda}(\hat{\mu}_{t_k}) + \frac{1}{2h} \W^2(\hat{\mu}_{t_k} , \mu_0) \right) \\
& \leq \liminf_{k \rightarrow \infty}  \mcf^{\nu}_{\lambda}(\mu) + \frac{1}{2h} \W^2(\mu , \mu_{t_k})   \\
& + \frac{1}{2h}  \W^2(\hat{\mu}_{t_k}, \mu_0) - \frac{1}{2h} \W^2(\hat{\mu}_{t_k}, \mu_{t_k})   \\
& = \mcf^{\nu}_{\lambda} (\mu) + \frac{1}{2h} \W^2(\mu, \mu_0) ,
\end{aligned}
\]
where the second inequality comes from the fact, that $\hat{\mu}_{t_k}$ minimizes $\mcf^{\nu}_{\lambda}(\cdot) + \frac{1}{2h}\W^2(\cdot, \mu_{t_k})$. From the above inequality and previously established facts, it follows that $\hat{\mu}$ is a minimum of $\mcg$ with density satisfying \ref{ineq:inf_norm_bound}.
\end{proof}

\begin{definition} \textbf{Minimizing movement scheme}
  \label{def:mini_move}
Let $r >0$ and  $\mcf : \mathbb{R_+} \times \PS(\cB(0,r)) \times \PS(\cB(0,r))\rightarrow \rset$ be a functional. Let $\mu_0 \in \PS(\cB(0,r))$ be a starting point. For $h> 0$ a piecewise constant trajectory $\mu^h : [0, \infty) \rightarrow \PS(\cB(0,r))$ for $\mcf$ starting at $\mu_0$ is a function such that:
\begin{itemize}
\item $\mu^h(0) = \mu_0$.
\item $\mu^h$ is constant on each interval $[nh, (n+1)h)$, so $\mu^h(t) = \mu^h(nh)$ with $n = \lfloor t/h \rfloor$.
\item $\mu^h((n+1)h )$ minimizes the  functional $ \zeta  \mapsto \mcf(h,  \zeta ,\mu^h(nh))$, for all $n \in \nset$.
\end{itemize}
We say $\hat{\mu}$ is a minimizing movement scheme for $\mcf$ starting at $\mu_0$, if there exists a family of piecewise constant trajectory $(\mu^h)_{h >0}$ for $\mcf$ such that $\hat{\mu}$ is a pointwise limit of $\mu^h$ as $h$ goes to $0$, \ie~for all $t \in \rset_+$, $\lim_{h \to 0} \mu^h(t) = \mu(t)$ in $\PS(\cB(0,r))$. We say that $\tilde{\mu}$ is a generalized minimizing movement for  $\mcf$ starting at $\mu_0$, if there exists a family of piecewise constant trajectory $(\mu^h)_{h >0}$ for $\mcf$ and  a sequence $(h_n)_n$, $\lim_{n \to \infty} h_n = 0$, such that $\mu^{h_n}$ converges pointwise to $\tilde{\mu}$.
\end{definition}

\begin{thm} \label{thm:existance_gmm_scheme}
Let $\nu$ be a probability measure on $\cB(0,1)$ with a strictly positive smooth density. Fix a regularization constant $\lambda > 0$ and radius $r > \sqrt{d}$. Given an absolutely continuous measure $\mu_0 \in \mathcal{P}(\cB(0,r))$ with density $\rho_0 \in \mrl^{\infty}(\cB(0,r))$, there is a generalized minimizing movement scheme $(\mu_t)_t$ in $\mathcal{P}(\cB(0,r))$ starting from $\mu_0$ for the functional defined by 
\begin{equation} \label{gmm:sw_ent_functional}
\mcf^{\nu}(h, \mu_+, \mu_-) = \mcf^{\nu}_{\lambda}(\mu_+) + \frac{1}{2h}\W^2(\mu_+, \mu_-).
\end{equation}
Moreover for any time $t > 0$, the probability measure $\mu_t = \mu(t)$ has density $\rho_t$ with respect to the Lebesgue measure and:
\begin{equation}
  \label{eq:bound:existance_gmm_scheme}
||\rho_t||_{\mrl^{\infty}} \leq e^{d t\sqrt{d}} ||\rho_0||_{\mrl^{\infty}} .  
\end{equation}
\end{thm}
\begin{proof}
We start by noting, that by \ref{thm:implicit_step} for any $h > 0$ there exists a piecewise constant trajectory $\mu^h$ for \ref{gmm:sw_ent_functional} starting at $\mu_0$. Furthermore for $t \geq 0$ measure $\mu_t^h = \mu^h(t)$ has density $\rho_t^h$, and:
\begin{equation} \label{ineq:dens_bound}
||\rho_t^h||_{\mrl^{\infty}} \leq e^{d \sqrt{d} (t+ h)} ||\rho_0||_{\mrl^{\infty}}.
\end{equation}
Let us choose $T > 0$. We denote $\rho^h(t,x) = \rho_t^h(x)$. 
For $ h \leq 1$, the functions $\rho^h$ lie in a ball in $\mrl^{\infty}([0,T] \times \cB(0,r))$, so from Banach-Alaoglu theorem there is a sequence $h_n$ converging to $0$, such that $\rho^{h_n}$ converges in weak-star topology in $\mrl^{\infty}([0,T] \times \cB(0,r))$ to a certain limit $\rho$. Since $\rho$ has to be nonnegative except for a set of measure zero, we assume $\rho$ is nonnegative. We denote $\rho_t(x) = \rho(t,x)$. We will prove that for almost all $t$, $\rho_t$ is a probability density and $\mu_{t}^{h_n}$ converges in $\W$ to a measure $\mu_t$ with density $\rho_t$.

First of all, for almost all $t \in [0,T]$, $\rho_t$ is a probability density, since for any Borel set $A \subseteq [0,T]$ the indicator of set $A \times \cB(0,r)$ is integrable, and hence by definition of the weak-star topology:
\[
\int_A \int_{\cB(0,r)} \rho_t(x) dx dt = \lim_{n \rightarrow \infty} \int_A \int_{\cB(0,r)} \rho_t^{h_n}(x) dx dt,
\]
and so we have to have $\int \rho_t(x) dx = 1$ for almost all $t \in [0,T]$. Nonnegativity of $\rho_t$ follows from nonnegativity of $\rho$. 

We will now prove, that for almost all $t \in [0,T]$ the measures $\mu_t^{h_n}$ converge to a measure with density $\rho_t$. Let $t \in (0,T)$, take $\delta < \min(T-t, t)$ and $\zeta \in \rmc^1(\cB(0,r))$. We have:
\begin{multline} \label{ineq:mu_t_conv_bound}
\left| \int_{\cB(0,r)} \zeta d\mu_t^{h_n} - \int_{\cB(0,r)} \zeta d\mu_t^{h_m} \right| \leq \\ \left| \int_{\cB(0,r)} \zeta d\mu_t^{h_n} - \frac{1}{2\delta} \int_{t - \delta}^{t + \delta} \int_{\cB(0,r)} \zeta d\mu_s^{h_n} ds \right| + 
\left| \int_{\cB(0,r)} \zeta d\mu_t^{h_m} - \frac{1}{2\delta} \int_{t - \delta}^{t + \delta} \int_{\cB(0,r)} \zeta d\mu_s^{h_m} ds \right| +  \\
\left| \frac{1}{2\delta} \int_{t - \delta}^{t + \delta} \int_{\cB(0,r)} \zeta d\mu_s^{h_m} ds - \frac{1}{2\delta} \int_{t - \delta}^{t + \delta} \int_{\cB(0,r)} \zeta d\mu_s^{h_n} ds \right|.
\end{multline}

Because $\mu_t^{h_n}$ have densities $\rho_t^{h_n}$ and both $\rho^{h_n}$, $\rho^{h_m}$ converge to $\rho$ in weak-star topology, the last element of the sum on the right hand side converges to zero, as $n,m \rightarrow \infty$. Next, we get a bound on the other two terms.

First, if we denote by $\gamma$ the optimal transport plan between $\mu_t^{h_n}$ and $\mu_s^{h_n}$, we have:
\begin{multline} \label{ineq:expectation_W2_bound}
\left| \int_{\cB(0,r)} \zeta d\mu_t^{h_n} - \int_{\cB(0,r)} \zeta d\mu_s^{h_n} \right|^2 \leq \int_{\cB(0,r) \times \cB(0,r)} \left| \zeta(x) - \zeta(y) \right|^2 d\gamma(x,y) \leq ||\nabla \zeta||_{\infty}^2 \W^2(\mu_t^{h_n}, \mu_s^{h_n}).
\end{multline}
In addition, for $n_t = \lfloor t/h_n \rfloor$ and $n_s = \lfloor s/h_n \rfloor$ we have $\mu_t^{h_n} = \mu_{n_t h_n}^{h_n}$ and $\mu_s^{h_n} = \mu_{n_s h_n}^{h_n}$. For all $k \geq 0$ we have:
\begin{equation}
  \label{eq:wasser_bound_1}
\W^2(\mu_{kh_n}^{h_n}, \mu_{(k+1)h_n}^{h_n}) \leq 2h_n(\mcf^{\nu}_{\lambda}(\mu_{kh_n}^{h_n}) - \mcf^{\nu}_{\lambda}(\mu_{(k+1)h_n}^{h_n})  .  
\end{equation}
Using this result and \eqref{ineq:expectation_W2_bound} and assuming without loss of generality $n_t \leq n_s$, from the Cauchy-Schwartz inequality we get:
\begin{align}
\nonumber
  \W^2(\mu_{t}^{h_n}, \mu_s^{h_n}) & \leq \left( \sum_{k= n_t}^{n_s-1} \W(\mu_{kh_n}^{h_n}, \mu_{(k+1)h_n}^{h_n}) \right)^2 \\
  \nonumber
                                   & \leq |n_t - n_s|\sum_{k=n_t}^{n_s 1} \W^2(\mu_{kh_n}^{h_n}, \mu_{(k+1)h_n}^{h_n}) \\
    \label{eq:wasser_bound_2}
& \leq 2h_n|n_t - n_s|(\mcf^{\nu}_{\lambda}(\mu_{n_t h_n}^{h_n}) - \mcf^{\nu}_{\lambda}(\mu_{n_s h_n}^{h_n}))  \leq 2C(|t-s| + h_n),
\end{align}
where we used for the last inequality, denoting $C = \mcf^{\nu}_{\lambda}(\mu_0) - \min_{\PS(\cB(0,r))} \mcf^{\nu}_{\lambda}$, that $(\mcf^{\nu}_{\lambda}(\mu_{kh_n}^{h_n}))_n$ is non-increasing by \eqref{eq:wasser_bound_1} and $\min_{\PS(\cB(0,r))} \mcf^{\nu}_{\lambda}$ is finite since $ \mcf^{\nu}_{\lambda}$ is lower semi-continuous.
Finally, using Jensen's inequality, the above bound and \ref{ineq:expectation_W2_bound} we get:
\[
\begin{aligned}
\left| \int_{\cB(0,r)} \zeta d\mu_t^{h_n} - \frac{1}{2\delta} \int_{t - \delta}^{t + \delta} \int_{\cB(0,r)} \zeta d\mu_s^{h_n} ds \right|^2 & \leq \frac{1}{2\delta} \int_{t - \delta}^{t + \delta} \left| \int_{\cB(0,r)} \zeta d\mu_{t}^{h_n} - \int_{\cB(0,r)} \zeta d\mu_s^{h_n} \right|^2 ds \\
& \leq \frac{C ||\nabla \zeta||_{\infty}^2}{\delta} \int_{t-\delta}^{t+\delta} (|t-s| +h_n) ds \\
& \leq 2C ||\nabla \zeta||_{\infty}^2 (h_n + \delta).
\end{aligned}
\]
Together with \eqref{ineq:mu_t_conv_bound}, when taking $\delta = h_n$, this result means that $\int_{\cB(0,r)} \zeta d\mu_t^{h_n}$ is a Cauchy sequence for all $t \in (0, T)$. On the other hand, since $\rho^{h_n}$ converges to $\rho$ in weak-star topology on $\mrl^{\infty}$, the limit of $\int_{\cB(0,r)} \zeta d\mu_t^{h_n}$ has to be $\int_{\cB(0,r)} \zeta(x) \rho_t(x) dx$ for almost all $t \in (0,T)$. This means that for almost all $t \in [0,T]$ sequence $\mu_{t}^{h_n}$ converges to a measure $\mu_t$ with density $\rho_t$.

Let $ S \in [0,T]$ be the set of times such that for $t \in S$ sequence $\mu_{t}^{h_n}$ converges to $\mu_t$. As we established almost all points from $[0,T]$ belong to $S$. Let $ t \in [0,T] \setminus S$. Then, there exists a sequence of times $t_k \in S$ converging to $t$, such that $\mu_{t_k}$ converge to some limit $\mu_t$. We have:
\[
\W(\mu_t^{h_n}, \mu_t) \leq \W(\mu_{t}^{h_n}, \mu_{t_k}^{h_n}) + \W(\mu_{t_k}^{h_n}, \mu_{t_k}) + \W(\mu_{t_k}, \mu_t).
\]
From which we have for all $k \geq 1$:
\[
\limsup_{n \rightarrow \infty} \W(\mu_{t}^{h_n}, \mu_t) \leq \W(\mu_{t_k}, \mu_t) + \limsup_{n\rightarrow \infty} \W(\mu_t^{h_n}, \mu_{t_k}^{h_n}),
\]
and using \eqref{eq:wasser_bound_2}, we get $\mu_{t}^{h_n} \rightarrow \mu_t$. Furthermore, the measure $\mu_t$ has to have density, since $\rho_{t}^{h_n}$ lie in a ball in $\mrl^{\infty}(\cB(0,r))$, so we can choose a subsequence of $\rho_t^{h_n}$ converging in weak-star topology to a certain limit $\hat{\rho}_t$, which is the density of $\mu_t$. 

We use now the diagonal argument to get convergence for all $t >0$. Let $(T_k)_{k=1}^{\infty}$ be a sequence of times increasing to infinity. Let $h_{n}^1$ be a sequence converging to $0$, such that $\mu_t^{h_n^1}$ converge to $\mu_t$ for all $t \in [0, T_1]$. Using the same arguments as above, we can choose a subsequence $h_n^2$ of $h_n^1$, such that $\mu_{t}^{h_n^2}$ converges to a limit $\mu_t$ for all $t \in [0, T_2]$. Inductively, we construct subsequences $h_{n}^k$, and in the end take $h_n = h_n^n$. For this subsequence we have that $\mu_t^{h_n}$ converges to $\mu_t$ for all $t > 0$, and $\mu_t$ has a density satisfying the bound from the statement of the theorem.

Finally, note that \eqref{thm:existance_gmm_scheme} follows from \eqref{ineq:dens_bound}.
\end{proof}

\begin{thm}
Let $(\mu_t)_{t \geq 0}$ be a generalized minimizing movement scheme given by Theorem~\ref{thm:existance_gmm_scheme} with initial distribution $\mu_0$ with density $\rho_0 \in \mrl(\cB(0,r))$. We denote by $\rho_t$ the density of $\mu_t$ for all $t \geq 0$. Then $\rho_t$ satisfies the continuity equation:
\[
\frac{\partial \rho_t}{\partial t} + \divop(v_t \rho_t) + \lambda \Delta \rho_t = 0 \,, \quad \quad \quad v_t(x) = - \int_{\mathbb{S}^{d-1}} \psi_{t, \theta}'(\langle x , \theta \rangle ) \theta d\theta ,
\]
in a weak sense, that is for all $\xi \in \rmc_c^{\infty} ([0, \infty)\times \cB(0,r))$ we have:
\[
\int_0^{\infty} \int_{\cB(0,r)} \left[\frac{\partial \xi}{\partial t}(t,x) - v_t \nabla \xi(t,x)  - \lambda \Delta \xi(t,x)\right] \rho_t(x) dx dt = -\int_{\cB(0,r)} \xi(0,x)\rho_0(x) dx.
\]
\end{thm}
\begin{proof}
Our proof is based on the proof of \cite{bonnotte2013unidimensional}[Theorem 5.6.1]. We proceed in five steps.

\begin{enumerate}[wide, labelwidth=!, labelindent=0pt,label=(\arabic*)]
\item Let $h_n \rightarrow 0$ be a sequence given by Theorem~\ref{thm:existance_gmm_scheme}, such that $\mu_t^{h_n}$ converges to $\mu_t$ pointwise. Furthermore we know that $\mu^{h_n}$ have densities $\rho^{h_n}$ that converge to $\rho$ in $\mrl^r$, for $r \geq 1$, and in weak-star topology in $\mrl^{\infty}$. Let $\xi \in \rmc_c^{\infty} ([0, \infty) \times \cB(0,r))$. We denote $\xi_{k}^n(x)  = \xi(kh_n, x)$. Using part $1$ of the proof of  \cite{bonnotte2013unidimensional}[Theorem 5.6.1], we obtain:
\begin{multline} \label{thm:cont_proof_part1}
\int_{\cB(0,r)} \xi(0, x) \rho_0(x) dx + \int_0^{\infty} \int_{\cB(0,r)} \frac{\partial \xi}{\partial t}(t,x) \rho_t(x) dx dt \\
= \lim_{n \rightarrow \infty} - h_n \sum_{k=1}^{\infty} \int_{\cB(0,r)} \xi_k^n(x) \frac{\rho_{kh_n}^{h_n}(x) - \rho_{(k-1)h_n}^{h_n}(x) }{h_n} dx.
\end{multline}
\item Again, this part is the same as part $2$ of the proof of  \cite{bonnotte2013unidimensional}[Theorem 5.6.1]. For any $\theta \in \mathbb{S}^{d-1}$ we denote by $\psi_{t, \theta}$ the unique Kantorovich potential from $\theta_{\#}^{*}\mu_t$ to $\theta_{\#}^{*}\nu$, and by $\psi_{t, \theta}^{h_n}$ the unique Kantorovich potential from $\theta_{\#}^{*} \mu_t^{h_n}$ to $\theta_{\#}^{*} \nu$. Then, by the same reasoning as part $2$ of the proof of  \cite{bonnotte2013unidimensional}[Theorem 5.6.1], we get:

\begin{multline} \label{thm:cont_proof_part2}
\int_0^{\infty} \int_{\cB(0,r)} \fint (\psi_{t, \theta})' (\langle \theta, x \rangle ) \langle \theta , \nabla \xi (x, t) \rangle d\theta d\mu_t(x) dt \\
= \lim_{n \rightarrow \infty} h_n \sum_{k=1}^{\infty} \int_{\cB(0,r)} \fint \psi_{kh_n, \theta}^{h_n} (\theta^{*}) \langle \theta, \nabla \xi_{k}^n \rangle d \theta d\mu_{kh_n}^{h_n}.
\end{multline}
\item  Since $\xi$ is compactly supported and smooth, $\Delta \xi$ is Lipschitz, and so for any $ t \geq 0$ if we take $k = \lfloor t/h_n \rfloor$ we get $| \Delta \xi_k^n(x) - \Delta \xi(t,x) | \leq C h_n$ for some constant $C$. Let $T > 0$ be such that $\xi(t,x) = 0 $ for $t > T$. We have:
\[
\left| \sum_{k=1}^{\infty} h_n \int_{\cB(0,r)} \Delta \xi_k^n(x) \rho_{kh_n}^{h_n}(x) dx - \int_{0}^{\plusinfty} \int_{\cB(0,r)} \Delta \xi(t,x) \rho_{t}^{h_n}(x) dx dt \right| \leq CTh_n .
\]
On the other hand,  we know, that $\rho^{h_n}$ converges to $\rho$ in weak star topology on $\mrl^{\infty}([0,T] \times \cB(0,r))$, and $\Delta \xi$ is bounded, so:
\[
\lim_{n \to \plusinfty}\left| \int_{0}^{\plusinfty} \int_{\cB(0,r)} \Delta \xi(t,x) \rho_{t}^{h_n}(x) dx dt - \int_{0}^{\plusinfty} \int_{\cB(0,r)} \Delta \xi (t,x) \rho_t(x) dx dt \right|= 0.
\]
Combining those two results give:
\begin{equation} \label{thm:cont_proof_part3}
\lim_{n \rightarrow \infty} h_n \sum_{k=1}^{\infty} \int_{\cB(0,r)} \Delta \xi_k^n(x) \rho_{kh_n}^{h_n}(x) dx = \int_{0}^{\plusinfty} \int_{\cB(0,r)} \Delta \xi (t,x) \rho_t(x) dx dt.
\end{equation}

\item 
  Let $\phi_{k}^{h_n}$ denote the unique Kantorovich potential from $\mu_{kh_n}^{h_n}$ to $\mu_{(k-1)h_n}^{h_n}$. Using  \cite{bonnotte2013unidimensional}[Propositions 1.5.7 and 5.1.7], as well as \cite{jordan1998variational}[Equation (38)] with $\Psi = 0$, and optimality of $\mu_{kh_n}^{h_n}$, we get:
\begin{multline} \label{thm:cont_proof_eq0}
\frac{1}{h_n} \int_{\cB(0,r)} \langle \nabla \phi_k^{h_n}(x) , \nabla \xi_{k}^n(x) \rangle d\mu_{kh_n}^{h_n}(x)   - \int_{\cB(0,r)} \fint (\psi_{kh_n}^{h_n})'(\theta^{*}) \langle \theta, \nabla \xi_k^n(x) \rangle d\theta d\mu_{kh_n}^{h_n}(x)\\ - \lambda \int_{\cB(0,r)}   \Delta \xi_k^n(x) d\mu_{kh_n}^{h_n}(x) ,
\end{multline}
which is the derivative of $\mcf^{\nu}_{\lambda}(\cdot) + \frac{1}{2h_n}\W^2(\cdot, \mu_{(k-1)h_n})$ in the direction given by vector field $\nabla \xi_k^n$ is zero.

Let $\gamma$ be the optimal transport between $\mu_{kh_n}^{h_n}$ and $\mu_{(k-1)h_n}^{h_n}$. Then:
\begin{equation} \label{thm:cont_proof_eq1}
\int_{\cB(0,r)} \xi_k^n(x) \frac{\rho_{kh_n}^{h_n}(x) - \rho_{(k-1)h_n}^{h_n}(x)}{h_n} dx = \frac{1}{h_n} \int_{\cB(0,r)} (\xi_k^n(y) - \xi_k^n(x)) d\gamma(x,y).
\end{equation}
\begin{equation} \label{thm:cont_proof_eq2}
\frac{1}{h_n}\int_{\cB(0,r)} \langle \nabla \phi_k^{h_n}(x) , \nabla \xi_k^n(x) \rangle d\mu_{kh_n}^{h_n}(x)  = \frac{1}{h_n} \int_{\cB(0,r)} \langle \nabla \xi_k^n(x), y-x \rangle d\gamma(x,y). 
\end{equation}
Since $\xi$ is $\rmc_c^{\infty}$, it has Lipschitz gradient. Let $C$ be twice the Lipschitz constant of $\nabla \xi$. Then we have $| \xi(y) - \xi(x) - \langle \nabla \xi(x), y-x \rangle | \leq C|x- y|^2$, and hence:
\begin{equation} \label{thm:cont_proof_eq3}
\int_{\cB(0,r)} |\xi_k^n(y) - \xi_k^n(x) - \langle \nabla \xi_k^n(x), y-x \rangle | d \gamma(x,y) \leq C\W^2( \mu_{(k-1)h_n}^{h_n}, \mu_{kh_n}^{h_n}).
\end{equation}
Combining \eqref{thm:cont_proof_eq1}, \eqref{thm:cont_proof_eq2} and \eqref{thm:cont_proof_eq3}, we get:
\begin{multline} \label{thm:cont_proof_eq4}
\left|\sum_{k=1}^{\infty} h_n \int_{\cB(0,r)} \xi_k^n(x) \frac{\rho_{kh_n}^{h_n} - \rho_{(k-1)h_n}^{h_n}}{h_n} dx  + 
\sum_{k=1}^{\infty} h_n\int_{\cB(0,r)} \langle \nabla \phi_k^{h_n} , \nabla \xi_k^n \rangle d\mu_{kh_n}^{h_n} \right| \\
\leq C\sum_{k=1}^{\infty} \W^2(\mu_{(k-1)h_n}^{h_n}, \mu_{kh_n}^{h_n}).
\end{multline}

As  some $\mcf^{\nu}_{\lambda}$ have a finite minimum on $\mathcal{P}(\cB(0,r))$, we have:
\begin{equation} \label{thm:cont_proof_eq5}
\begin{aligned}
\sum_{k=1}^{\infty} \W^2(\mu_{(k-1)h_n}^{h_n}, \mu_{kh_n}^{h_n}) & \leq  2 h_n \sum_{k=1}^{\infty} \mcf^{\nu}_{\lambda}(\mu_{(k-1)h_n}^{h_n}) - \mcf^{\nu}_{\lambda}(\mu_{kh_n}^{h_n}) \\ & \leq 2h_n \left(\mcf^{\nu}_{\lambda}(\mu_0) - \min_{\mathcal{P}(\cB(0,r))}\mcf^{\nu}_{\lambda} \right).
\end{aligned}
\end{equation}
and so the sum on the right hand side of the equation goes to zero as $n$ goes to infinity.

From \eqref{thm:cont_proof_eq4}, \eqref{thm:cont_proof_eq5} and \eqref{thm:cont_proof_eq0} we conclude:
\begin{multline} \label{thm:cont_proof_part4}
\lim_{n \rightarrow \infty} - h_n \sum_{k=1}^{\infty} \xi_k^n(x)\frac{\rho_{kh_n}^{h_n} - \rho_{(k-1)h_n}^{h_n}}{h_n} dx = \\
\lim_{n \rightarrow \infty} \left( h_n \sum_{k=1}^{\infty} \int_{\cB(0,r)} \fint \psi_{kh_n, \theta}^{h_n} (\theta^{*}) \langle \theta, \nabla \xi_{k}^n \rangle d \theta d\mu_{kh_n}^{h_n} + h_n \sum_{k=1}^{\infty} \int_{\cB(0,r)} \Delta \xi_k^n(x) \rho_{kh_n}^{h_n}(x) dx  \right),
\end{multline}
where both limits exist, since the difference of left hand side and right hand side of the equation goes to zero, while the left hand side converges to a finite value by \eqref{thm:cont_proof_part1}.
\item  Combining \eqref{thm:cont_proof_part1}, \eqref{thm:cont_proof_part2}, \eqref{thm:cont_proof_part3} and \eqref{thm:cont_proof_part4} we get the result.   
\end{enumerate}
\end{proof}

\section{Proof of Theorem~\ref{thm:euler}}

Before proceeding to the proof, let us first define the following Euler-Maruyama scheme which will be useful for our analysis:
\begin{align}
\hat{X}_{k+1}  = \hat{X}_k + h \hat{v}(\hat{X}_k, \mu_{kh}) + \sqrt{2\lambda h}Z_{n+1},
\end{align}
where $\mu_t$ denotes the probability distribution of $X_t$ with $(X_t)_t$ being the solution of the original SDE \eqref{eqn:sde}. Now, consider the probability distribution of $\hat{X}_k$ as $\muh_{kh}$.  Starting from the discrete-time process $(\hat{X}_k)_{k\in \mathbb{N}_+}$, we first define a continuous-time process $(Y_t)_{t\geq 0}$ that linearly interpolates $(\hat{X}_k)_{k\in \mathbb{N}_+}$, given as follows: 
\begin{align}
d Y_t = \tilde{v}_t(Y) dt + \sqrt{2 \lambda} dW_t, \label{eqn:sde_linear}
\end{align}
where $\tilde{v}_t(Y) \triangleq - \sum_{k=0}^{\infty} \hat{v}_{kh} (Y_{kh}) \mathds{1}_{[kh, (k+1)h)}(t)$ and $\mathds{1}$ denotes the indicator function.
Similarly, we define a continuous-time process $(U_t)_{t\geq 0}$ that linearly interpolates $(\bar{X}_k)_{k\in \mathbb{N}_+}$, defined by \eqref{eqn:euler_asymp}, given as follows: 
\begin{align}
d U_t = \bar{v}_t(U) dt + \sqrt{2 \lambda} dW_t, \label{eqn:sde_linear2}
\end{align}
where
$\bar{v}_t(U) \triangleq - \sum_{k=0}^{\infty} \hat{v} (U_{kh},
\mub_{kh}) \mathds{1}_{[kh, (k+1)h)}(t)$ and $\mub_{kh}$ denotes the
probability distribution of $\bar{X}_k$.  Let us denote the
distributions of $(X_t)_{t \in [0,T]}$, $(Y_t)_{t \in [0,T]}$ and
$(U_t)_{t \in \ccint{0,T}}$ as $\pi_{X}^T$, $\pi_{Y}^T$ and $\pi_{U}^T$
respectively with $T = Kh$.

\newcommand{\minvsp}{0}

We consider the following assumptions: \vspace{\minvsp pt}
\begin{assumption}
\label{asmp:sde_ergo}
For all $\lambda >0$, the SDE  \eqref{eqn:sde} has a unique strong solution denoted by $(X_t)_{t\geq 0}$ for any starting point $x \in \R^d$. %
\vspace{\minvsp pt}
\end{assumption}
\begin{assumption}
\label{asmp:lipschitz}
There exits $L < \infty$ such that
\begin{align}
\| v_t(x) - v_{t'}(x') \| \leq L ( \|x-x' \| + |t-t'|),
\end{align}
where $v_t(x) = v(x,\mu_t)$ and
\begin{align}
\| \hat{v}(x,\mu) - \hat{v}(x',\mu') \| \leq L ( \|x-x' \| + \|\mu-\mu'\|_{\TV}).
\end{align}
\vspace{\minvsp pt}
\end{assumption}
\begin{assumption}
\label{asmp:dissip}
For all $t \geq 0$, $v_t$ is dissipative, i.e. for all $x \in \R^d$,
\begin{align}
\langle x, v_t(x) \rangle \geq m \|x\|^2 -b,
\end{align}
for some $m,b >0$.
\vspace{\minvsp pt}
\end{assumption}
\begin{assumption}
\label{asmp:stochgrad}
The estimator of the drift satisfies the following conditions: \ $\E[\hat{v}_t] = v_t$ for all $t \geq 0$, and for all $t\geq 0$, $x \in \R^d$,
\begin{align}
\E[ \|\hat{v}(x,\mu_t) - v(x,\mu_t) \|^2] \leq 2 \delta(L^2 \|x\|^2 + B^2),
\end{align}
for some $\delta \in (0,1)$. %
\vspace{\minvsp pt}
\end{assumption}
\begin{assumption}
\label{asmp:init_fun}
For all $t \geq 0$: $|\Psi_t(0)| \leq A$ and $\|v_t(0)\| \leq B$,
for $A,B \geq 0$, where $\Psi_t = \int_{\mathbb{S}^{d-1}} \psi_{t}(\ps{\theta}{\cdot}) d \theta$. 
\end{assumption}

We start by upper-bounding $\| \muh_{Kh} - \mu_T \|_{\TV}$. 
\begin{lemma}
\label{lem:euler}
Assume that the conditions \Cref{asmp:lipschitz,asmp:stochgrad,asmp:dissip,asmp:init_fun} hold. Then, the following bound holds:
\begin{align}
\| \muh_{Kh} - \mu_{T} \|_{\TV}^2\leq \| \pi^T_{Y} - \pi_{X}^T \|_{\TV}^2 \leq \frac{L^2 K}{4\lambda} \Bigl( \frac{C_1 h^3}{3} + 3 \lambda d h^2 \Bigr) + \frac{C_2 \delta K h}{8\lambda},
\end{align}
where $C_1 \triangleq 12(L^2 C_0 + B^2)+1$, $C_2 \triangleq 2 (L^2 C_0 + B^2)$, $C_0 \triangleq C_e +2  (1 \vee \frac1{m})(b+2B^2 + d \lambda)$, and $C_e$ denotes the entropy of $\mu_0$.
\end{lemma}
\begin{proof}
We use the proof technique presented in \cite{dalalyan2017theoretical,raginsky17a}.  It is easy to verify that for all $k \in \mathbb{N}_+$, we have $Y_{kh} = \hat{X}_k$. 

By Girsanov's theorem to express the Kullback-Leibler (KL) divergence between these two distributions, given as follows:
\begin{align}
\KL (\pi_{X}^T || \pi_{Y}^T) &= \frac1{4 \lambda} \int_0^{Kh} \E[ \|v_t(Y_t) + \tilde{v}_t(Y) \|^2 ]  \> dt \\
&= \frac1{4 \lambda} \sum_{k=0}^{K-1} \int_{kh}^{(k+1)h} \E[ \|v_t(Y_t) + \tilde{v}_t(Y) \|^2 ] \> dt \\
&= \frac1{4 \lambda} \sum_{k=0}^{K-1} \int_{kh}^{(k+1)h} \E[ \|v_t(Y_t) - \hat{v}_{kh}(Y_{kh}) \|^2 ] \> dt.
\end{align}
By using $v_t(Y_t) - \hat{v}_{kh}(Y_{kh}) = ( v_t(Y_t) - v_{kh}(Y_{kh})) + ( v_{kh}(Y_{kh}) - \hat{v}_{kh}(Y_{kh}))$, we obtain
\begin{align}
\nonumber \KL (\pi_{X}^T || \pi_{Y}^T) \leq& \frac1{2 \lambda} \sum_{k=0}^{K-1} \int_{kh}^{(k+1)h} \E[ \|v_t(Y_t) - {v}_{kh}(Y_{kh}) \|^2 ] \> dt \\
&+  \frac1{2 \lambda} \sum_{k=0}^{K-1} \int_{kh}^{(k+1)h} \E[ \|v_{kh}(Y_{kh}) - \hat{v}_{kh}(Y_{kh}) \|^2 ] \> dt \\
\nonumber \leq& \frac{L^2}{\lambda} \sum_{k=0}^{K-1} \int_{kh}^{(k+1)h} \bigl(\E[ \|Y_t - Y_{kh} \|^2 ] + (t-kh)^2 \bigr)  \> dt \\
&+  \frac1{2 \lambda} \sum_{k=0}^{K-1} \int_{kh}^{(k+1)h} \E[ \|v_{kh}(Y_{kh}) - \hat{v}_{kh}(Y_{kh}) \|^2 ] \> dt . \label{eqn:lem1_proof_interm}
\end{align}
The last inequality is due to the Lipschitz condition \Cref{asmp:lipschitz}.

Now, let us focus on the term $\E[ \|Y_t - Y_{kh} \|^2]$. By using \eqref{eqn:sde_linear}, we obtain:
\begin{align}
Y_t - Y_{kh} = - (t-kh) \hat{v}_{kh}(Y_{kh}) + \sqrt{2 \lambda (t-kh)} Z,
\end{align}
where $Z$ denotes a standard normal random variable. By adding and subtracting the term $-(t-kh) v_{kh}(Y_{kh})$, we have:
\begin{align}
Y_t - Y_{kh} = -(t-kh)v_{kh}(Y_{kh}) + (t-kh)(v_{kh}(Y_{kh}) - \hat{v}_{kh}(Y_{kh})) + \sqrt{2 \lambda (t-kh)} Z.
\end{align}
Taking the square and then the expectation of both sides yields:
\begin{align}
\nonumber \E[ \|Y_t - Y_{kh} \|^2] \leq& 3(t-kh)^2 \E[ \|v_{kh}(Y_{kh})\|^2] + 3 (t-kh)^2 \E[\|v_{kh}(Y_{kh}) - \hat{v}_{kh}(Y_{kh})\|^2] \\
&+ 6\lambda (t-kh)d.
\end{align}
As a consequence of \Cref{asmp:lipschitz} and \Cref{asmp:init_fun}, we have $\| v_t(x)\| \leq L\|x\|+B$ for all $t \geq 0$, $x\in \R^d$. Combining this inequality with \Cref{asmp:stochgrad}, we obtain:
\begin{align}
\nonumber \E[ \|Y_t - Y_{kh} \|^2] \leq& 6(t-kh)^2 (L^2 \E[ \|Y_{kh}\|^2] + B^2) + 6(t-kh)^2 (L^2 \E[ \|Y_{kh}\|^2] + B^2) \\
&+ 6\lambda (t-kh)d\\
=& 12(t-kh)^2 (L^2 \E[ \|Y_{kh}\|^2] + B^2) + 6\lambda (t-kh)d.
\end{align}
By Lemma 3.2 of \cite{raginsky17a}\footnote{Note that Lemma 3.2 of \cite{raginsky17a} considers the case where the drift is not time- or measure-dependent. However, with \Cref{asmp:dissip} it is easy to show that the same result holds for our case as well.}, we have $\E[ \|Y_{kh}\|^2] \leq C_0 \triangleq C_e +2  (1 \vee \frac1{m})(b+2B^2 + d \lambda)$, where $C_e$ denotes the entropy of $\mu_0$. Using this result in the above equation yields:
\begin{align}
\E[ \|Y_t - Y_{kh} \|^2] \leq& 12(t-kh)^2 (L^2 C_0 + B^2) + 6\lambda (t-kh)d. \label{eqn:lem_bound1}
\end{align}

We now focus on the term $\E[ \|v_{kh}(Y_{kh}) - \hat{v}_{kh}(Y_{kh}) \|^2 ]$ in \eqref{eqn:lem1_proof_interm}. Similarly to the previous term, we can upper-bound this term as follows:
\begin{align}
\E[ \|v_{kh}(Y_{kh}) - \hat{v}_{kh}(Y_{kh}) \|^2 ] \leq& 2 \delta(L^2 \E[\|Y_{kh}\|^2] + B^2) \\
\leq& 2 \delta(L^2 C_0 + B^2). \label{eqn:lem_bound2}
\end{align}

By using \eqref{eqn:lem_bound1} and \eqref{eqn:lem_bound2} in \eqref{eqn:lem1_proof_interm}, we obtain:
\begin{align}
\nonumber \KL (\pi_{X}^T || \pi_{Y}^T) \leq& \frac{L^2}{\lambda} \sum_{k=0}^{K-1} \int_{kh}^{(k+1)h} \bigl(12(t-kh)^2 (L^2 C_0 + B^2) + 6\lambda (t-kh)d +(t-kh)^2 \bigr) dt\\
&+  \frac1{2 \lambda} \sum_{k=0}^{K-1} \int_{kh}^{(k+1)h} 2 \delta(L^2 C_0 + B^2) \> dt \\
=& \frac{L^2 K}{\lambda} \Bigl( \frac{C_1 h^3}{3} + \frac{6 \lambda d h^2}{2} \Bigr) + \frac{C_2 \delta K h}{2\lambda},
\end{align}
where $C_1 = 12(L^2 C_0 + B^2)+1$ and $C_2 = 2 (L^2 C_0 + B^2)$.

Finally, by using the data processing and Pinsker inequalities, we obtain:
\begin{align}
\| \muh_{Kh} - \mu_{T} \|_{\TV}^2 \leq \| \pi_{X}^T - \pi_{Y}^T \|_{\TV}^2 \leq& \frac1{4} \KL (\pi_{X}^T || \pi_{Y}^T) \\
=& \frac{L^2 K}{4\lambda} \Bigl( \frac{C_1 h^3}{3} + 3 \lambda d h^2 \Bigr) + \frac{C_2 \delta K h}{8\lambda}.
\end{align}
This concludes the proof.
\end{proof}

Now, we bound the term $\| \mub_{Kh} - \muh_{Kh} \|_{\TV}$.
\begin{lemma}
\label{lem:euler2}
Assume that \Cref{asmp:lipschitz} holds. Then the following bound holds:
\begin{align}
\| \pi_{U}^T - \pi_{Y}^T \|_{\TV}^2  \leq \frac{L^2 K h}{16 \lambda}  \|\pi_{X}^T - \pi_{U}^T \|_{\TV}^2 .
\end{align}
\end{lemma}
\begin{proof}
We use that same approach than in Lemma~\ref{lem:euler}. By Girsanov's theorem once again, we have
\begin{align}
\KL (\pi_{Y}^T || \pi_{U}^T) &= \frac1{4 \lambda} \sum_{k=0}^{K-1} \int_{kh}^{(k+1)h} \E[ \|\hat{v}(U_{kh}, \mu_{kh}) - \hat{v}(U_{kh},\mub_{kh}) \|^2 ] \> dt,
\end{align}
where $\pi_U^T$ denotes the distributions of $(U_t)_{t \in [0,T]}$ with $T = Kh$. By using \Cref{asmp:lipschitz}, we have:
\begin{align}
\KL (\pi_{Y}^T || \pi_{U}^T) &\leq \frac{L^2 h}{4 \lambda} \sum_{k=0}^{K-1} \|\mu_{kh} - \mub_{kh} \|_{\TV}^2   \\
&\leq \frac{L^2 K h}{4 \lambda}  \|\pi_{X}^T - \pi_{U}^T \|_{\TV}^2  .
\end{align}
By applying the data processing and Pinsker inequalities, we obtain the desired result.
\end{proof}

\subsection{Proof of Theorem~\ref{thm:euler}}

Here, we precise the statement of Theorem~\ref{thm:euler}.

\begin{thm}
\label{lem:euler3}
Assume that the assumptions in Lemma~\ref{lem:euler} and Lemma~\ref{lem:euler2} hold. Then for $\lambda > \frac{KL^2h}{8}$, the following bound holds:
\begin{align}
\|\mub_{Kh} - \mu_{T} \|_{\TV}^2 &\leq \delta_\lambda \Biggl\{ \frac{L^2 K}{2\lambda} \Bigl( \frac{C_1 h^3}{3} + 3 \lambda d h^2 \Bigr) + \frac{C_2 \delta K h}{4\lambda} \Biggr\},
\end{align}
where $\delta_\lambda = (1 -\frac{KL^2h}{8\lambda})^{-1} $.
\end{thm}
\begin{proof}
We have the following decomposition: (with $T= Kh$)
\begin{align}
\|\pi_{X}^T - \pi_{U}^T \|_{\TV}^2 &\leq 2 \|\pi_{X}^T - \pi_Y^T \|_{\TV}^2 + 2\|\pi_Y^T - \pi_{U}^T \|_{\TV}^2 \\
&\leq  \frac{L^2 K}{2\lambda} \Bigl( \frac{C_1 h^3}{3} + 3 \lambda d h^2 \Bigr) + \frac{C_2 \delta K h}{4\lambda} +  \frac{L^2 K h}{8 \lambda}  \|\pi_X^T - \pi_{U}^T \|_{\TV}^2 \\ 
&\leq \Bigl(1 -\frac{KL^2h}{8\lambda} \Bigr)^{-1} \Biggl\{ \frac{L^2 K}{2\lambda} \Bigl( \frac{C_1 h^3}{3} + 3 \lambda d h^2 \Bigr) + \frac{C_2 \delta K h}{4\lambda} \Biggr\}.
\end{align}
The second line follows from Lemma~\ref{lem:euler} and Lemma~\ref{lem:euler2}. Last line follows from the assumption that $\lambda$ is large enough. This completes the proof.
\end{proof}

\section{Proof of Corollary~\ref{coro:precision}}
\begin{proof}
Considering the bound given in Theorem~\ref{thm:euler}, the choice $h$ implies that
\begin{align}
\frac{\delta_\lambda L^2 K}{2\lambda} \Bigl( \frac{C_1 h^3}{3} + 3 \lambda d h^2 \Bigr) \leq \varepsilon^2. \label{eqn:cor_Th}
\end{align}
This finalizes the proof. 
\end{proof}

\section{Additional Experimental Results}

\subsection{The Sliced Wasserstein Flow}
The whole code for the Sliced Wasserstein Flow was implemented in Python, for use with Pytorch\footnote{\url{http://www.pytorch.org}.}. The code was written so as to run efficiently on GPU, and is available on the publicly available repository related to this paper\footnote{\url{https://github.com/aliutkus/swf}.}.

In practice, the SWF involves relatively simple operations, the most important being:
\begin{itemize}
  \item For each random $\theta\in\{\theta_n\}_{n=1\dots N_\theta}$,  compute its inner product with all items from a dataset and obtain the empirical quantiles for these \emph{projections}.
  \item At each step $k$ of the SWF, for each projection $z=\left<\theta, \bar{X}^i_k\right>$, apply two piece-wise linear functions, corresponding to the scalar optimal transport $\psi_{k, \theta}'(z)$.
\end{itemize}
Even if such steps are conceptually simple, the quantile and required linear interpolation functions were not available on GPU for any framework we could figure out at the time of writing this paper. Hence, we implemented them ourselves for use with Pytorch, and the interested reader will find the details in the Github repository dedicated to this paper.

Given these operations, putting a SWF implementation together is straightforward.
The code provided allows not only to apply it on any dataset, but also provides routines to have the computation of these sketches running in the background in a parallel manner.

\subsection{The need for dimension reduction through autoencoders}

In this study, we used an autoencoder trained on the dataset as a dimension reduction technique, so that the SWF is applied to transport particles in a latent space of dimension $d\approx 50$, instead of the original $d>1000$ of image data.

The curious reader may wonder why SWF is not applied directly to this original space, and what performances should be expected there. We have done this experiment, and we found out that SWF has much trouble rapidly converging to satisfying samples. In figure~\ref{fig:suppnoae}, we show the progressive evolution of particles undergoing SWF when the target is directly taken as the uncompressed dataset.

\begin{figure}
\centering
\includegraphics[width=\columnwidth]{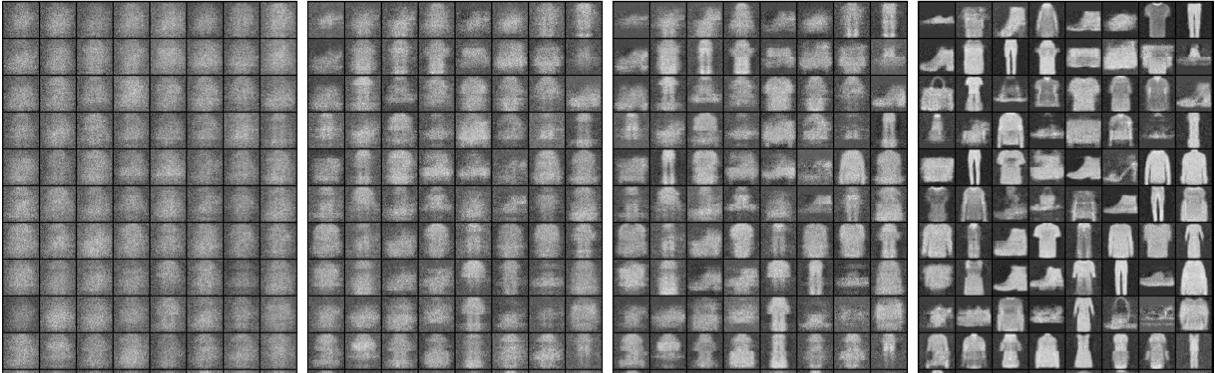}
\includegraphics[width=\columnwidth]{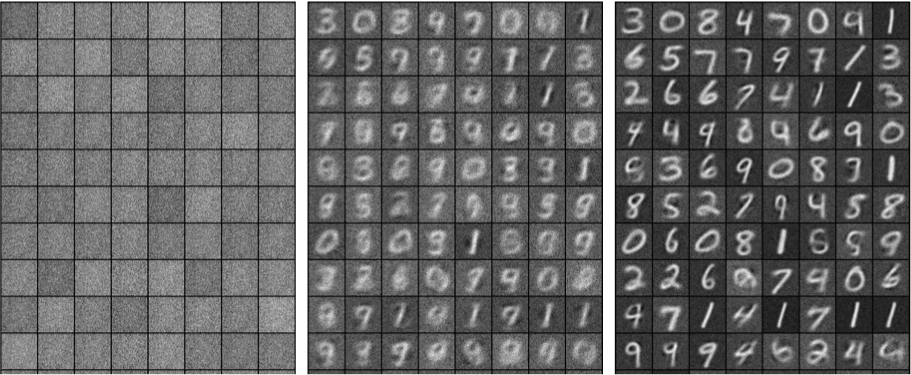}
\caption{The evolution of SWF through 15000 iterations, when the original high-dimensional data is kept instead of working on reduced bottleneck features as done in the main document. Showing results on the MNIST and FashionMNIST datasets. For a visual comparison for FashionMNIST, we refer the reader to \cite{samangouei2018defensegan}.}
\label{fig:suppnoae}
\end{figure}

In this experiment, the strategy was to change the projections $\theta$ at each iteration, so that we ended up with a set of projections being $\{\theta_{n,k}\}_{n=1\dots N_\theta}^{k=1\dots K}$ instead of the fixed set of $N_\theta$ we now consider in the main document (for this, we picked $N_\theta=200$). This strategy is motivated by the complete failure we observed whenever we picked such fixed projections throughout iterations, even for a relatively large number as $N_\theta=16000$.

As may be seen on Figure~\ref{fig:suppnoae}, the particles definitely converge to samples from the desired datasets, and this is encouraging. However, we feel that the extreme number of iterations required to achieve such convergence comes from the fact that theory needs an integral over the $d-$dimensional sphere at each step of the SWF, which is clearly an issue whenever $d$ gets too large.
Although our solution of picking new samples from the sphere at each iteration alleviated this issue to some extent, the curse of dimensionality prevents us from doing much better with just thousands of \emph{random} projections at a time.

This being said, we are confident that good performance would be obtained if millions of random projections could be considered for transporting such high dimensional data because i/ theory suggests it and ii/ we observed excellent performance on reduced dimensions.

However, we, unfortunately, did not have the computing power it takes for such large scale experiments and this is what motivated us in the first place to introduce some dimension-reduction technique through AE.

\subsection{Structure of our autoencoders for reducing data dimension}

As mentioned in the text, we used autoencoders to reduce the dimensionality of the transport problem. The structure of these networks is the following:

\begin{itemize}
  \item $\textbf{Encoder}$ Four 2d convolution layers with (num\_chan\_out, kernel\_size, stride, padding) being $(3,3,1,1)$, $(32,2,2,0)$, $(32,3,1,1)$, $(32, 3,1,1)$, each one followed by a ReLU activation. At the output, a linear layer gets the desired bottleneck size.

  \item $\textbf{Decoder}$ A linear layer gets from the bottleneck features to a vector of dimension $8192$, which is reshaped as $(32, 16,16)$. Then, three convolution layers are applied, all with $32$ output channels and (kernel\_size, stride, panning) being respectively $(3,1,1)$, $(3,1,1)$, $(2,2,0)$. A 2d convolution layer is then applied with an output number of channels being that of the data ($1$ for black and white, $3$ for color), and a (kernel\_size, stride, panning) as $(3,1,1)$. In any case, all layers are followed by a ReLU activation, and a sigmoid activation is applied a the very output.
\end{itemize}

Once these networks defined, these autoencoders are trained in a very simple manner by minimizing the binary cross entropy between input and output over the training set of the considered dataset (here MNIST, CelebA or FashionMNIST). This training was achieved with the Adam algorithm \cite{kingma2014adam} with learning rate $1e-3$.

No additional training trick was involved as in Variational Autoencoder \cite{kingma2013VAE} to make sure the distribution of the bottleneck features matches some prior. The core advantage of the proposed method in this respect is indeed to turn any previously learned AE as a generative model, by automatically and non-parametrically transporting particles drawn from an arbitrary prior distribution $\mu$ to the observed empirical distribution $\nu$ of the bottleneck features over the training set.

\subsection{Convergence plots of SWF}

\begin{figure}
\begin{centering}
\includegraphics[width=0.5\columnwidth]{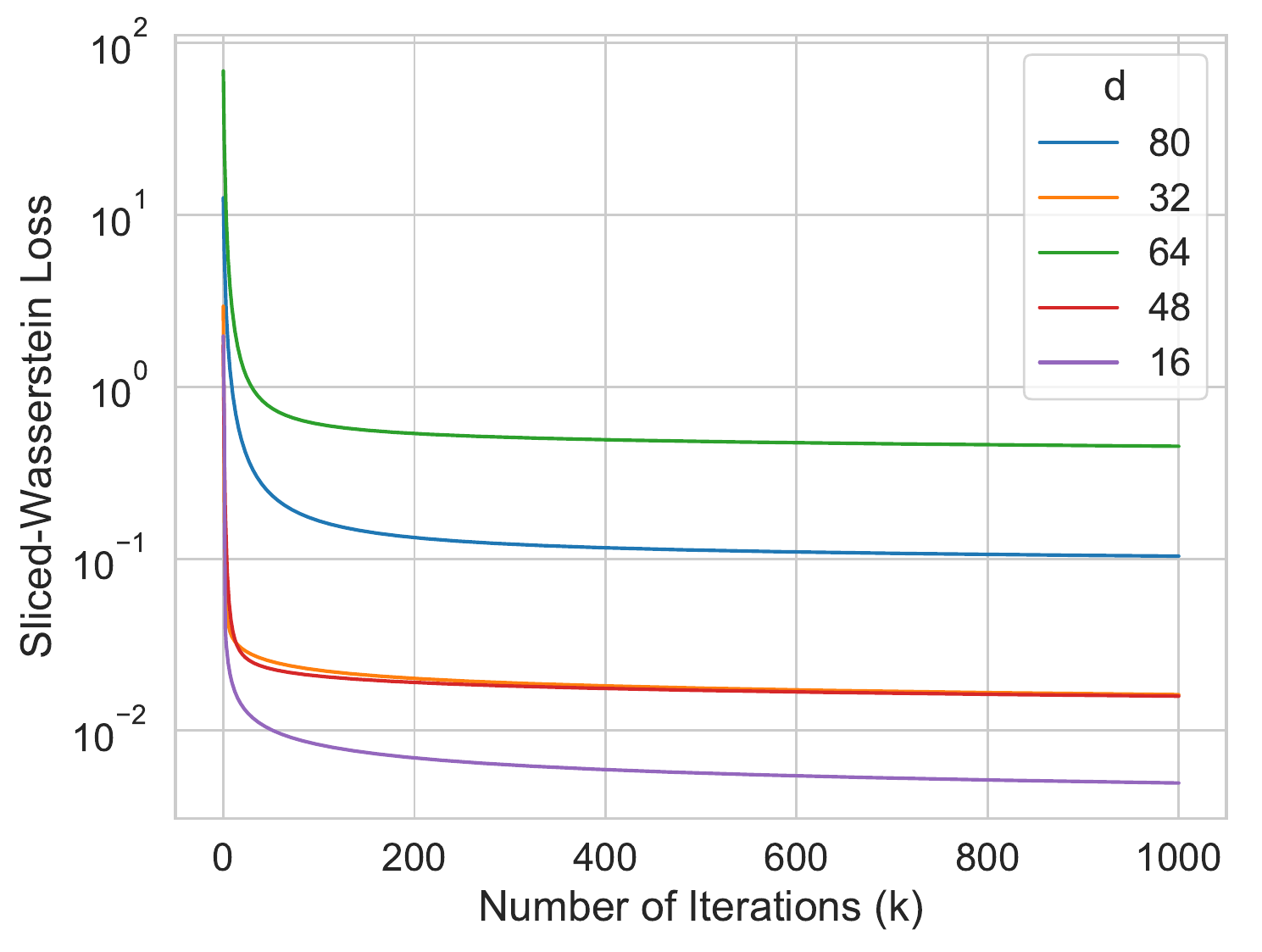}
\par\end{centering}
\caption{Approximately computed $\SW$ between the output $\bar{\mu}_{k}^{N}$ and data distribution $\nu$ in the MNIST experiment for different dimensions $d$ for the bottleneck features (and the corresponding pre-trained AE).
\label{fig:supptoy_sw}}
\end{figure}

In the same experimental setting as in the main document, we also illustrate the behavior of the algorithm for varying dimensionality $d$ for the bottleneck-features. To monitor the convergence of SWF as predicted by theory, we display the approximately computed $\SW$ distance between the distribution of the particles and the data distribution. Even though minimizing this distance is not the real objective of our method, arguably, it is still a good proxy for understanding the convergence behavior.

Figure~\ref{fig:supptoy_sw} illustrates the results. We observe that, for all choices of $d$, we see a steady and smooth decrease in the cost for all runs, which is in line with our theory. The absolute value of the cost for varying dimensions remains hard to interpret at this stage of our investigations.

\section{Additional samples}

\subsection{Evolution throughout iterations}

In Figures~\ref{fig:suppmnist} and \ref{fig:suppfmnist} below, we provide the evolution of the SWF algorithm on the Fashion MNIST and the MNIST datasets in higher resolution, for an AE with $d=48$ bottleneck features.

\newcommand{\picwidth}{0.2}%
\begin{figure}
\centering
\includegraphics[width=\picwidth\columnwidth]{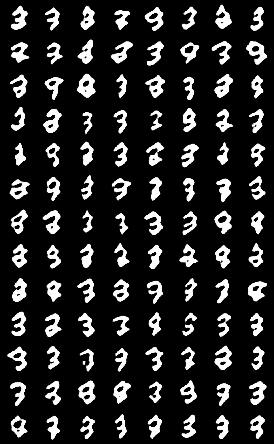}
\includegraphics[width=\picwidth\columnwidth]{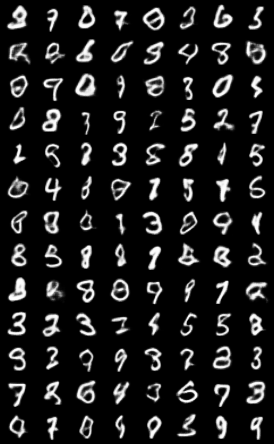}
\includegraphics[width=\picwidth\columnwidth]{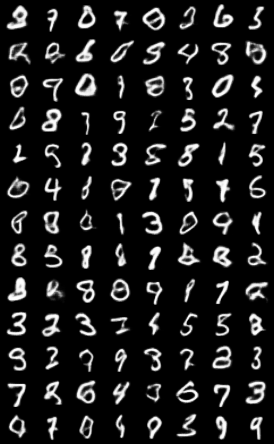}
\includegraphics[width=\picwidth\columnwidth]{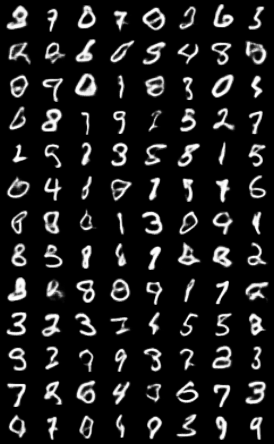}\\
\includegraphics[width=\picwidth\columnwidth]{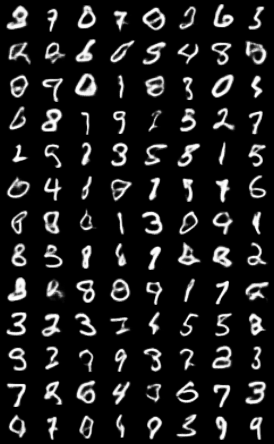}
\includegraphics[width=\picwidth\columnwidth]{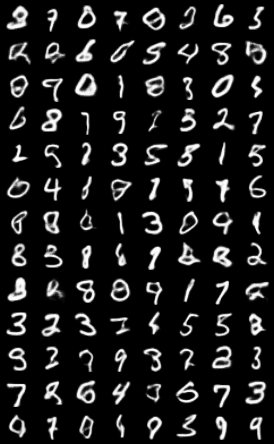}
\includegraphics[width=\picwidth\columnwidth]{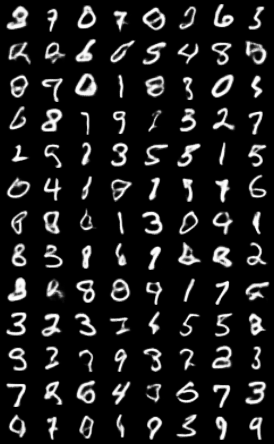}
\includegraphics[width=\picwidth\columnwidth]{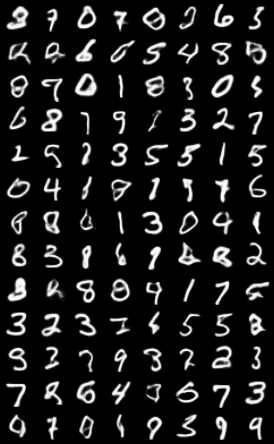}
\caption{The evolution of SWF through 200 iterations on the MNIST dataset. Plots are for $1$, $11$, $21$, $31$, $41$, $51$, $101$ and $201$ iterations}
\label{fig:suppmnist}
\end{figure}

\begin{figure}
\centering
\includegraphics[width=\picwidth\columnwidth, frame]{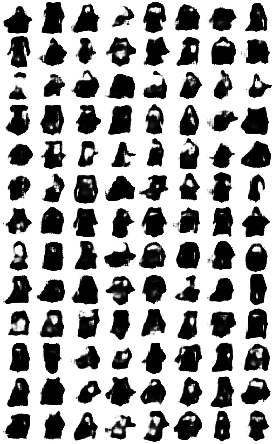}%
\includegraphics[width=\picwidth\columnwidth, frame]{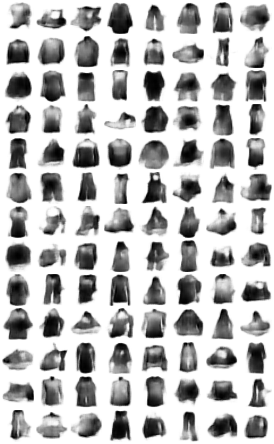}%
\includegraphics[width=\picwidth\columnwidth, frame]{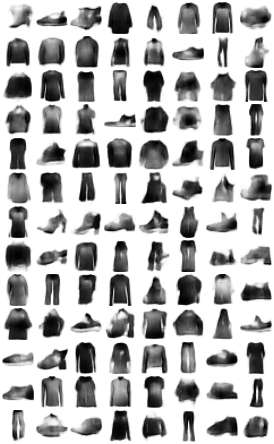}%
\includegraphics[width=\picwidth\columnwidth, frame]{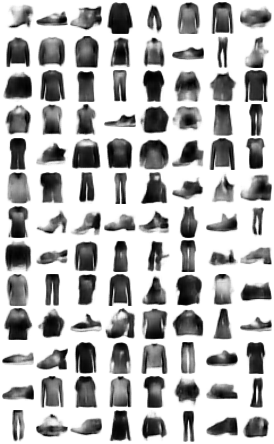}\\%
\includegraphics[width=\picwidth\columnwidth, frame]{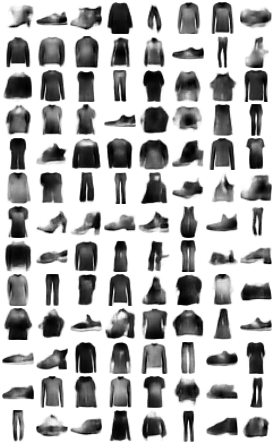}%
\includegraphics[width=\picwidth\columnwidth, frame]{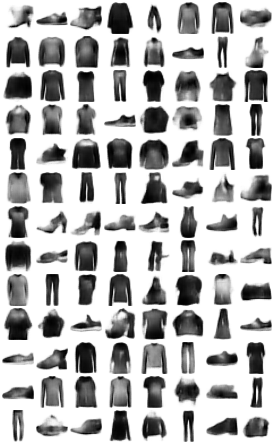}%
\includegraphics[width=\picwidth\columnwidth, frame]{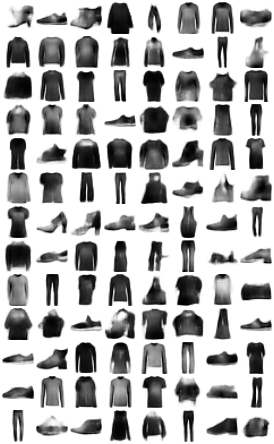}%
\includegraphics[width=\picwidth\columnwidth, frame]{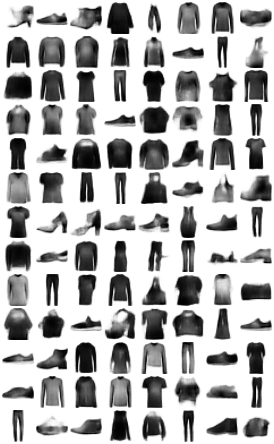}%
\caption{The evolution of SWF through 200 iterations on the FashionMNIST dataset. Plots are for $1$, $11$, $21$, $31$ (upper row) and $41$, $51$, $101$, $201$ (lower row) iterations}
\label{fig:suppfmnist}
\end{figure}

\subsection{Training samples, interpolation and extrapolation}

In Figures~\ref{fig:suppmnistsamples} and \ref{fig:suppfmnistsamples} below, we provide other examples of outcome from SWF, both for the MNIST and the FashionMNIST datasets, still with $d=48$ bottleneck features.

The most noticeable fact we may see on these figures is that while the actual particles which went through SWF, as well as linear combinations of them, all yield very satisfying results, this is however not the case for particles that are drawn randomly and then brought through a pre-learned SWF.

Once again, we interpret this fact through the curse of dimensionality: while we saw in our toy GMM example that using a pre-trained SWF was totally working for small dimensions, it is already not so for $d=48$ and only $3000$ training samples.

This noticed, we highlight that this generalization weakness of SWF for high dimensions is not really an issue, since it is always possible to i/ run SWF with more training samples if generalization is required ii/ re-run the algorithm for a set of new particles. Remember indeed that this does not require passing through the data again, since the distribution of the data projections needs to be done only once.
\renewcommand{\picwidth}{0.3}%

\begin{figure}
\centering
\subfigure[particles undergoing SWF]{
\includegraphics[width=\picwidth\columnwidth]{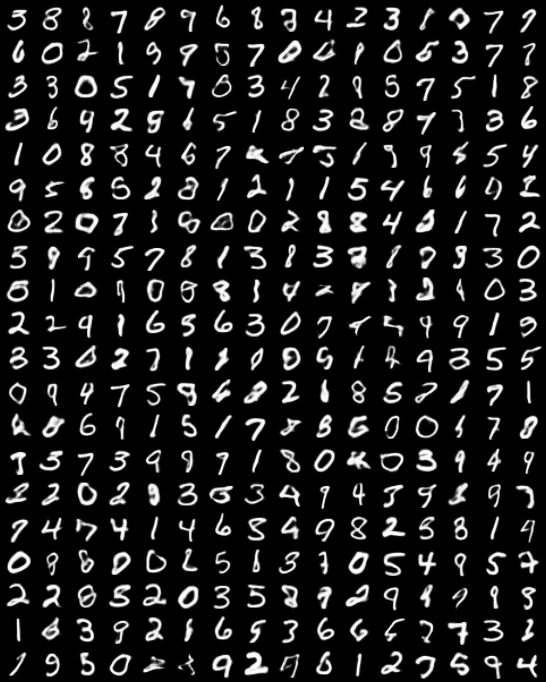}}
\subfigure[After SWF is done: applying learned map on linear combinations of train particles]{
\includegraphics[width=\picwidth\columnwidth]{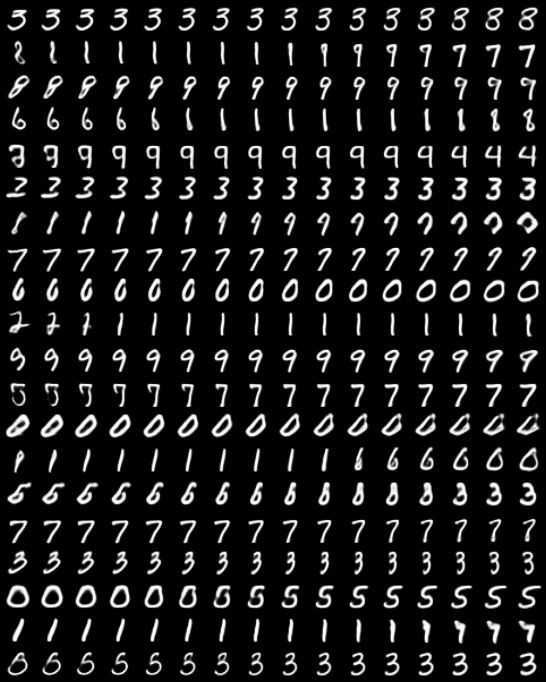}}
\subfigure[After SWF is done: applying learned map on random inputs.]{
\includegraphics[width=\picwidth\columnwidth]{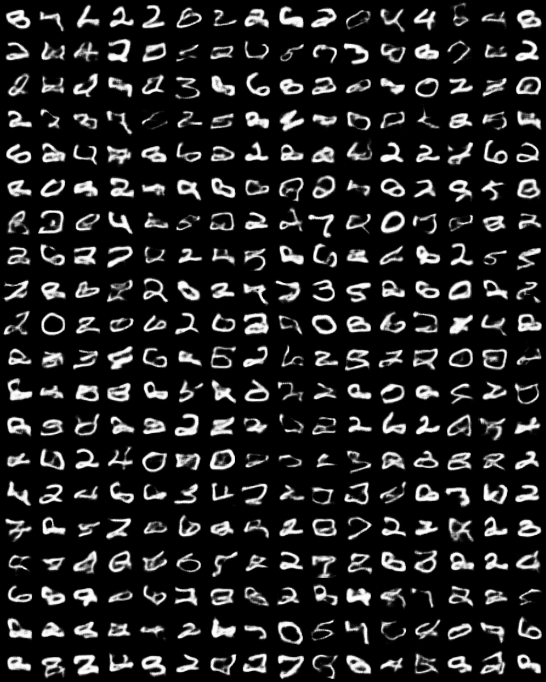}}
\caption{SWF on MNIST: training samples, interpolation in learned mapping, extrapolation.}
\label{fig:suppmnistsamples}
\end{figure}

\begin{figure}
\centering
\subfigure[particles undergoing SWF]{
\includegraphics[width=\picwidth\columnwidth]{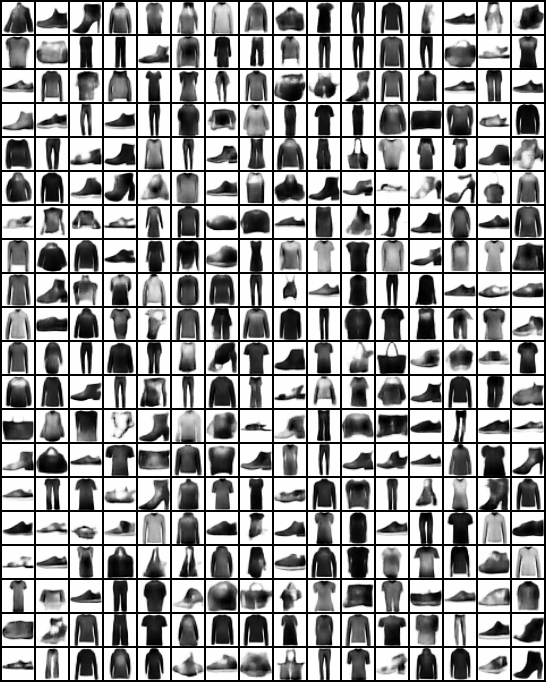}}
\subfigure[After SWF is done: applying learned map on linear combinations of train particles]{
\includegraphics[width=\picwidth\columnwidth]{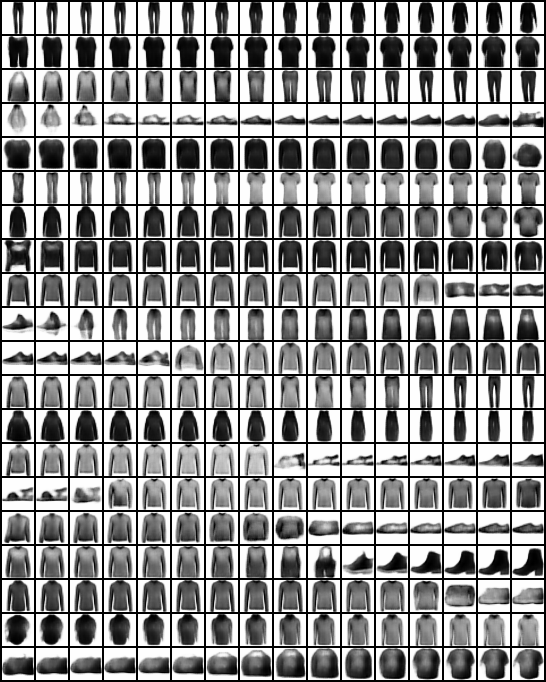}}
\subfigure[After SWF is done: applying learned map on random inputs.]{
\includegraphics[width=\picwidth\columnwidth]{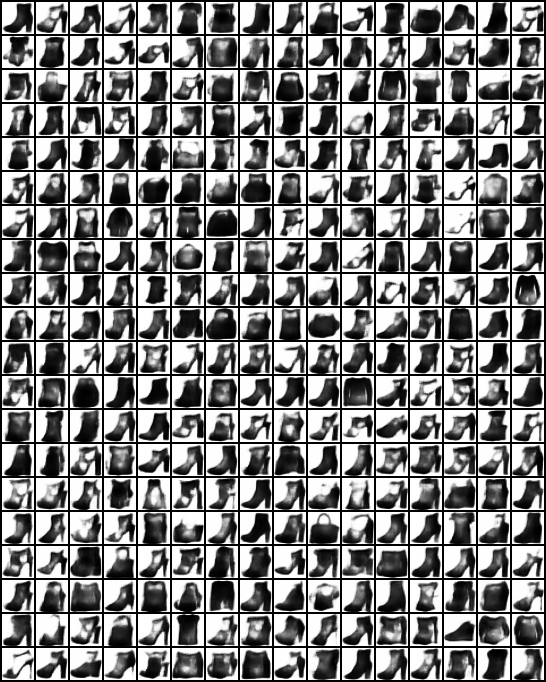}}
\caption{SWF on FashionMNIST: training samples, interpolation in learned mapping, extrapolation.}
\label{fig:suppfmnistsamples}
\end{figure}


\end{document}